\documentclass[acmlarge,screen,nonacm,natbib=false]{acmart}

\usepackage{tikz}

\usepackage{tikz}
\usepackage{pgfplots}
\pgfplotsset{compat=1.18}

\usepackage{amsfonts}
\usepackage{footmisc}

\usepackage[utf8]{inputenc}
\usepackage{xcolor}
\usepackage{booktabs}
\usepackage{amsthm}
\usepackage{listings}

\usepackage{pgfplotstable}
\usepackage{longtable}

\usepackage{pgfplots}
\usepackage{amsmath}
\usepackage{microtype}
\usepackage{soul}
\usepackage{enumerate}%

\newcommand{\name}{\textsc{shrinker}}
\newcommand{\popper}{\textsc{Popper}}
\newcommand{\ale}{\textsc{Aleph}}

\newcommand{\ilasp}{\textsc{ilasp}}

\theoremstyle{definition}
\newtheorem{definition}{Definition}
\newtheorem{example}{Example}
\newtheorem{theorem}{Theorem}
\newtheorem{proposition}{Proposition}
\newtheorem{lemma}{Lemma}
\newtheorem{statement}{Statement}

\usepackage{algorithm}
\usepackage{algorithmic}

\usepackage{fancyvrb}

\lstnewenvironment{myalgorithm}[1][] 
{
    \lstset{ 
        showspaces=false,               
        showstringspaces=false,         
        mathescape=true,
        numbers=left,
        escapeinside={*}{*},
        columns=flexible,
        numbersep=2pt,        
        keywordstyle=\bfseries,
        keywords={,and, return, not, def, in, if, else, for, foreach, while, }
        numbers=left,
        xleftmargin=.0\textwidth,
        #1 
    }
}
{}

\setcopyright{cc}
\acmDOI{10.1613/jair.1.21708}

\RequirePackage[
  datamodel=acmdatamodel,
  style=acmauthoryear,
  backend=biber,
  giveninits=true,
  uniquename=init
  ]{biblatex}

\begin{document}

\title{Honey, I Shrunk the Hypothesis Space (Through Logical Preprocessing)}
\subtitle{Published in JAIR, Vol. 85, Art. 48, 2026. DOI: 10.1613/jair.1.21708}
\author{Andrew Cropper}
\authornote{Corresponding Author.}
\email{andrew.cropper@helsinki.fi}
\orcid{0000-0002-4543-7199}

\affiliation{%
  \institution{ELLIS Institute}
  \country{Finland}
}

\affiliation{%
  \institution{University of Helsinki}
  \country{Finland}
}

\author{Filipe Gouveia}
\orcid{0000-0003-1852-2782}
\email{jfgouveia@ciencias.ulisboa.pt}
\affiliation{%
  \institution{LASIGE, Informática, Faculdade de Ciências, Universidade de Lisboa}
  \country{Portugal}
  }

\author{David M. Cerna}
\orcid{0000-0002-6352-603X}
\email{dcerna@cs.cas.cz}
\email{david.cerna@dynatrace.com}

\affiliation{%
  \institution{Czech Academy of Sciences Institute of Computer Science}
  \country{Czechia}
}

\affiliation{%
  \institution{Dynatrace Research}
  \country{Austria}
}

\renewcommand{\shortauthors}{Cropper, Gouveia \& Cerna}

\begin{abstract}
\noindent
Inductive logic programming (ILP) is a form of logical machine learning.
The goal is to search a hypothesis space for a hypothesis that generalises training examples and background knowledge.
We introduce an approach that \emph{shrinks} the hypothesis space before an ILP system searches it.
Our approach uses background knowledge to find rules that cannot be in an optimal hypothesis regardless of the training examples.
For instance, our approach discovers relationships such as \emph{even numbers cannot be odd} and \emph{prime numbers greater than 2 are odd}.
It then removes violating rules from the hypothesis space. 
We implement our approach using answer set programming and use it to shrink the hypothesis space of a constraint-based ILP system.
Our experiments on multiple domains, including visual reasoning and game playing, show that our approach can substantially reduce learning times whilst maintaining predictive accuracies.
For instance, given just 10 seconds of preprocessing time, our approach can reduce learning times from over 10 hours to only 2 seconds.
\end{abstract}


\received{22 January 2026}
\received[accepted]{18 March 2026}

\maketitle

\section{Introduction}
\label{sec:intro}

Inductive logic programming (ILP) \cite{mugg:ilp,ilpintro} is a form of logical machine learning.
The goal is to search a hypothesis space for a hypothesis that generalises training examples and background knowledge (BK).
The distinguishing feature of ILP is that it uses logical rules to represent hypotheses, examples, and BK.

As with all forms of machine learning, a major challenge in ILP is searching vast hypothesis spaces.
To illustrate this challenge, consider a learning task where we want to generalise examples of an arbitrary relation \emph{h}.
Assume we can build rules using the unary relations \emph{odd}, \emph{even}, and \emph{int} and the binary relations \emph{head}, \emph{tail}, \emph{len}, \emph{lt}, and \emph{succ}.
The \emph{rule space} is the set of all possible rules and 
contains rules such as:

\begin{center}
\begin{tabular}{l}
\emph{r$_1$  = h $\leftarrow$ len(A,B)}\\
\emph{r$_2$  = h $\leftarrow$ tail(A,A)}\\
\emph{r$_3$  = h $\leftarrow$ tail(A,B), tail(B,A)}\\
\emph{r$_4$  = h $\leftarrow$ head(A,B), head(A,C)}\\
\emph{r$_5$  = h $\leftarrow$ tail(A,B), tail(B,C), tail(A,C)}\\
\emph{r$_6$  = h $\leftarrow$ tail(A,A), head(A,B), odd(B)}\\
\emph{r$_7$  = h $\leftarrow$ head(A,B), odd(B), 
even(B)}\\
\emph{r$_8$  = h $\leftarrow$ head(A,B), int(B), 
odd(B)}\\
\emph{r$_9$  = h $\leftarrow$ head(A,B), succ(B,C), succ(C,D), lt(B,D)}
\end{tabular}
\end{center}

\noindent
The hypothesis space is the powerset of the rule space, so it can be enormous.
In this simple scenario, if we allow rules to contain at most four literals and four variables, there are at least 3,183,545 possible rules and thus $2^{3,183,545}$ hypotheses. 

Most research tackles this challenge by developing techniques to efficiently search large hypothesis spaces \cite{foil,progol,tilde,aleph,atom,quickfoil,mugg:metagold,popper}.
In other words, most research assumes a fixed hypothesis space and focuses on efficiently searching that space.

Rather than develop a novel search algorithm, we introduce a way to \emph{shrink} the hypothesis space before giving it to an ILP system.
The idea is to use the given BK to find rules that cannot be in an optimal hypothesis regardless of the concept we want to learn, where an optimal hypothesis minimises a cost function on the training data, such as finding a minimal description length hypothesis \cite{maxsynth}.
We call such rules \emph{pointless rules}.

To illustrate our idea, consider the previous scenario.
Assume, for simplicity, that the rule space contains only those nine rules, rather than all 3,183,545 rules. 
Also assume that we have BK with only the facts:

\begin{center}
\begin{lstlisting}[basicstyle=\ttfamily, frame=single]
head(ijcai,i)  head(ecai,e)  head(cai,c)  
tail(ijcai,jcai)  tail(jcai,cai)  tail(ecai,cai) tail(cai,ai)  tail(ai,i)  
len(ijcai,5)  len(jcai,4)  len(cai,3)  len(ecai,4)  len(ai,2)  len(i,1)  
int(1)  int(2)  int(3)  int(4)  
succ(1,2)  succ(2,3)  succ(3,4)  
even(2)  even(4)  
odd(1)  odd(3)  
lt(1,2)  lt(1,3)  lt(1,4)  lt(2,3)  lt(2,4)  lt(3,4)  
\end{lstlisting}
\end{center}

\noindent
We use this BK to find four types of pointless rules: \emph{unsatisfiable}, \emph{implication reducible}, \emph{recall reducible}, and \emph{singleton reducible}.

An \emph{unsatisfiable} rule can never be true regardless of the concept we want to learn, i.e. irrespective of specific training examples.
For instance, given this BK, if we adopt a closed world assumption \cite{cwa}, then, because there is no fact of the form \emph{tail(A,A)}, we can deduce that \emph{tail/2} is irreflexive and remove rules $r_2$ and $r_6$ from the rule space because their bodies are unsatisfiable.
The rule space is now:

\begin{center}
\begin{tabular}{l}
\emph{r$_1$  = h $\leftarrow$ len(A,B)}\\
\emph{r$_3$  = h $\leftarrow$ tail(A,B), tail(B,A)}\\
\emph{r$_4$  = h $\leftarrow$ head(A,B), head(A,C)}\\
\emph{r$_5$  = h $\leftarrow$ tail(A,B), tail(B,C), tail(A,C)}\\
\emph{r$_7$  = h $\leftarrow$ head(A,B), odd(B), 
even(B)}\\
\emph{r$_8$  = h $\leftarrow$ head(A,B), int(B), 
odd(B)}\\
\emph{r$_9$  = h $\leftarrow$ head(A,B), succ(B,C), succ(C,D), lt(B,D)}
\end{tabular}
\end{center}

\noindent
Similarly, we can deduce that \emph{tail/2} is asymmetric and antitransitive and remove rules $r_3$ and $r_5$ from the rule space.
The rule space is now:

\begin{center}
\begin{tabular}{l}
\emph{r$_1$  = h $\leftarrow$ len(A,B)}\\
\emph{r$_4$  = h $\leftarrow$ head(A,B), head(A,C)}\\
\emph{r$_7$  = h $\leftarrow$ head(A,B), odd(B), 
even(B)}\\
\emph{r$_8$  = h $\leftarrow$ head(A,B), int(B), 
odd(B)}\\
\emph{r$_9$  = h $\leftarrow$ head(A,B), succ(B,C), succ(C,D), lt(B,D)}
\end{tabular}
\end{center}

\noindent
Finally, we can deduce that \emph{odd/1} and \emph{even/1} are mutually exclusive and remove rule $r_7$.
The rule space is now:

\begin{center}
\begin{tabular}{l}
\emph{r$_1$  = h $\leftarrow$ len(A,B)}\\
\emph{r$_4$  = h $\leftarrow$ head(A,B), head(A,C)}\\
\emph{r$_8$  = h $\leftarrow$ head(A,B), int(B), 
odd(B)}\\
\emph{r$_9$  = h $\leftarrow$ head(A,B), succ(B,C), succ(C,D), lt(B,D)}
\end{tabular}
\end{center}

\noindent
An \emph{implication reducible} rule contains an implied literal.
For instance, because \emph{odd(A)} implies \emph{int(A)}, any rule with both \emph{int(A)} and \emph{odd(A)} in the body has a redundant literal so cannot be in an optimal hypothesis.
Therefore, we can remove rule $r_8$ from the rule space.
The rule space is now:

\begin{center}
\begin{tabular}{l}
\emph{r$_1$  = h $\leftarrow$ len(A,B)}\\
\emph{r$_4$  = h $\leftarrow$ head(A,B), head(A,C)}\\
\emph{r$_9$  = h $\leftarrow$ head(A,B), succ(B,C), succ(C,D), lt(B,D)}
\end{tabular}
\end{center}

\noindent
Similarly, if \emph{succ(B,C)} and \emph{succ(C,D)} both hold then 
\emph{lt(B,D)} must hold because $D=B+2$.
Therefore, any rule with these three literals has a redundant literal so cannot be in an optimal hypothesis.
Therefore, we can remove rule $r_9$ from the rule space.
The rule space is now:

\begin{center}
\begin{tabular}{l}
\emph{r$_1$  = h $\leftarrow$ len(A,B)}\\
\emph{r$_4$  = h $\leftarrow$ head(A,B), head(A,C)}
\end{tabular}
\end{center}

\noindent
A \emph{recall reducible} rule contains a redundant literal determined by the number of ground instances deducible from the BK.
For instance, because \emph{head/2} is functional,  any rule with both \emph{head(A,B)} and \emph{head(A,C)} is reducible because $B=C$.
Therefore, we can remove rule $r_4$ from the rule space.
The rule space is now:

\begin{center}
\begin{tabular}{l}
\emph{r$_1$  = h $\leftarrow$ len(A,B)}\\
\end{tabular}
\end{center}

\noindent
Finally, a \emph{singleton reducible} rule contains a literal that is always true.
For instance, if $A$ is always a list then the literal \emph{len(A,B)} in rule $r_1$ is always true because every list has a length and the variable $B$ is unconstrained.
In other words, because the variable $B$ only appears once in this rule, the literal \emph{len(A,B)} is redundant because it has no discriminatory power.
Therefore, we can remove rule $r_1$ from the hypothesis space.
The rule space is now empty.

We show that we can remove hypotheses that contain any of these types of pointless rules from the hypothesis space without removing optimal hypotheses.
As this scenario shows, removing hypotheses with pointless rules can substantially shrink the hypothesis space before even looking at the given training examples.
In other words, our approach can discover where not to search before searching for a hypothesis.

To find unsatisfiable and implication reducible rules, we use answer set programming (ASP) \cite{asp}.
Specifically, we build small rules and use ASP to check whether these rules are unsatisfiable or implication reducible with respect to the BK.
We repeat this process until reaching a user-defined timeout.
The output of this stage is a set of unsatisfiable and implication reducible rules.

To find recall reducible rules, we use a bottom-up approach 
similar to \citet{savnik1993bottom} for discovering functional dependencies in a database and \citet{queryoptimisation} for building efficient queries.
For a background relation and any subsequence of its arguments, we calculate the maximum number of answer substitutions.
For instance, in the BK given above, \emph{succ/2} has at most three answer substitutions because there are only three ground instances.
Likewise, for the literal \emph{succ(A,B)}, if $A$ is ground then there is at most one answer substitution for $B$.
The output of this stage is a set of recall reducible rules.

To find singleton reducible rules we use ASP to determine whether a relation is partial or total and use that information to deduce singleton reducible rules.
The output of this stage is a set of singleton reducible rules.

We use the set of pointless rules to shrink the hypothesis space of a constraint-based ILP system \cite{popper,maxsynth}.
For instance, if we discover that \emph{even} and \emph{odd} are mutually exclusive, we build constraints that remove rules that contain both the body literals \emph{odd(A)} and \emph{even(A)}.
Likewise, if we discover that the literals \emph{succ(B,C)}, \emph{succ(C,D)}, and \emph{lt(B,D)} are implication reducible, the constraints remove rules that contain all these literals in them.
Our constraints remove non-optimal hypotheses from the hypothesis space so that an ILP system never considers them when searching for a hypothesis.

Our shrinking approach has many benefits.
The approach is system-agnostic and could be used by any ILP system.
For instance, \ale{}'s rule pruning mechanism \cite{aleph} could use the constraints.
The approach is a stand-alone preprocessing step, independent of a specific learning algorithm or learning task.
Therefore, we can use our approach as a preprocessing step on a set of BK shared by many tasks.
The main benefit is a drastic reduction in learning times without reducing predictive accuracy.
For instance, given only 10 seconds to shrink the hypothesis space, our approach can reduce learning time from over 10 hours to just a few seconds.

\subsubsection*{Novelty and Contributions}
The key novelty of this paper is an approach that automatically shrinks the hypothesis space of an ILP system as a preprocessing step by finding pointless rules.
The impact is vastly reduced learning times, demonstrated in diverse domains, sometimes reducing learning times by 99\%.

This paper substantially extends our earlier work \cite{discopopper}.
In that preliminary work, we introduced the idea of using BK to discover constraints on hypotheses, such as that a number cannot be both even and odd.
The preliminary work used a human-provided predefined set of properties, such as \emph{antitransitivity} and \emph{irreflexivity}.
In this work, we generalise the idea to four types of pointless rules, where implication reducible and singleton reducible are new.
Almost all the content in this paper is new.

Overall, our contributions are:

\begin{itemize}
    \item We introduce the hypothesis space reduction problem, where the goal is to find a subset of a hypothesis space which still contains all optimal hypotheses.
    \item We define pointless rules, i.e. unsatisfiable, implication reducible, recall reducible, and singleton reducible rules.
    We show that hypotheses with pointless rules cannot be optimal (Propositions \ref{prop:sound_unsat2}, \ref{prop:sound_sat3}, \ref{prop:sound_sat4}, and \ref{prop:sound_sat5}).
    \item We describe a hypothesis space shrinking approach and implement it in a system called \name{}.
    We prove that \name{} only removes non-optimal hypotheses from the hypothesis space (Proposition \ref{prop:correctness}).
    \item We use \name{} to bootstrap a constraint-based ILP system \cite{popper,maxsynth}.
    \item We experimentally show on multiple domains that, given only 10 seconds preprocessing time, \name{} can drastically reduce ILP learning times, sometimes from over 10 hours to 2 seconds.
\end{itemize}

\section{Related Work}
\label{sec:related}

\subsection{Program Synthesis}
ILP is a form of program synthesis, where the goal is to automatically generate computer programs from examples.
This topic, which \citet{gulwani2017program} consider the holy grail of AI, interests a broad community \cite{dilp,ellis:scc}.
Although our approach could be applied to any form of program synthesis, it is well-suited to ILP because ILP's logical representation naturally supports declarative knowledge through logical constraints.

\subsection{Rule Induction}
ILP approaches induce rules from data, similar to rule learning methods \cite{DBLP:conf/ruleml/FurnkranzK15}.
It is difficult to compare ILP methods with recent rule mining techniques, such as AMIE+ \cite{DBLP:journals/vldb/GalarragaTHS15} and RDFRules \cite{rdfrules}.
Most rule-mining methods are limited to unary and binary relations and require facts as input. 
They also typically operate under an open-world assumption. 
By contrast, ILP usually operates under a closed-world assumption, supports relations of any arity, and can learn from definite programs as background knowledge.

\subsection{ILP}
Our approach automatically shrinks the hypothesis space by finding rules that cannot be in an optimal hypothesis.
Many systems allow humans to manually specify conditions for when a rule cannot be in a hypothesis \cite{progol,aleph,ilasp}.
Most systems only reason about the conditions \emph{after} constructing a hypothesis, such as \ale{}'s rule pruning mechanism \cite{aleph}.
By contrast, we automatically discover constraints as a preprocessing step and remove rules that violate them from the hypothesis space \emph{before} searching for a hypothesis.

\subsection{Bottom Clauses}
Many ILP systems, such as \ale{}, use bottom clauses \cite{progol} or variants, such as kernel sets \cite{xhail}, to restrict the hypothesis space. 
The bottom clause of an example $e$ is the most specific clause that entails $e$.
Bottom clauses can also be seen as informing an ILP system where to search because an ILP system never needs to consider a rule more specific than the bottom clause.
Our approach is similar, as it restricts the hypothesis space.
There are, however, many differences.
Bottom clauses are example specific.
To find a rule to cover an example, a learner constructs the bottom clause for that specific example, which it uses to bias the search.
Building a bottom clause can be expensive and the resulting bottom clause can be very large \cite{progol}.
Moreover, as it contains all literals that can be true, a bottom clause will likely include many redundant literals, such as \emph{int(A)} and \emph{odd(A)}.
By contrast, our approach finds such reducible rules and removes them from the hypothesis space.
In addition, in the worst case, a learner needs to build a bottom clause for every positive training example.
Indeed, the Kernel set, as used by \textsc{xhail} \cite{xhail}, is built from all positive examples at once.
This approach is expensive when there are many examples.
By contrast, our bias discovery approach is task independent and only uses the BK, not the training examples.
In other words, the running time of our approach is independent of the number of examples.
Because of this difference, we can reuse any discovered bias across examples and tasks.
For instance, if we discover that the successor relation (\emph{succ}) is asymmetric, we can reuse this bias across multiple tasks.
Moreover, bottom clauses introduce several fundamental limitations, such as making it difficult for the corresponding ILP system to learn recursive hypotheses and hypotheses with predicate invention \cite{ilpintro}.
By contrast, our approach does not consider examples and only looks at the given BK to reduce the hypothesis space so it can be used by systems that can learn recursive hypotheses and perform predicate invention.
Another difference is that bottom clauses are essential for inverse entailment methods \cite{progol}.
By contrast, our approach is an independent preprocessing step that allows us to amortise its cost (such as by using a timeout).
Finally, our shrinking approach could help bottom clause approaches, such as to bootstrap a constraint-based inverse entailment approach \cite{atom} or to provide constraints for \ale{}'s clause pruning conditions \cite{aleph}.
 
\subsection{Redundancy in ILP}
Most work in ILP focuses on improving search efficiency for a fixed hypothesis space.
There is little work on hypothesis space reduction.
\citet{DBLP:conf/ilp/SrinivasanK05} introduce methods to reduce the dimensionality of bottom clauses using statistical methods by compressing bottom clauses.
The authors state that the resulting lower dimensional space translates directly to a smaller hypothesis space. 
We differ because we do not use examples in the preprocessing step and do not need to build bottom clauses, alleviating all the issues in the aforementioned bottom clause section.
\citet{DBLP:conf/ilp/FonsecaCSC04} define \emph{self-redundant} clauses, similar to our definition of an implication reducible rule (Definition \ref{def:breducible}).
Their definition does not guarantee that a redundant clause's refinements (specialisations) are redundant.
By contrast, we prove that specialisations of an implication reducible rule are reducible (Proposition \ref{prop:specRedu}).
Moreover, the authors do not propose detecting such rules.
Instead, they expect users to provide information about them.
By contrast, we introduce \name{}, which automatically finds reducible and indiscriminate rules.
\citet{RaedtR04} check whether a rule has a redundant atom before testing it on examples.
If so, it avoids the coverage check.
This approach requires \textit{anti-monotonic} constraints, where if a constraint holds for a rule then it holds for all its generalisations, but not necessarily its specialisations.
By contrast, we find properties that allow us to prune specialisations of a rule.
Moreover, the approach of \citet{RaedtR04} does not explicitly identify implications between literals and thus could keep building rules with the same implied literals.
By contrast, \name{} explicitly finds implications between literals to prune the hypothesis space.
\citet{quickfoil} prune rules with simple forms of syntactic redundancy.
For instance, for the rule \emph{h(X) $\leftarrow$ p(X,Y), p(X,Z)}, the authors detect that \emph{p(X,Z)} is duplicate to the literal \emph{p(X,Y)} under the renaming $Z \mapsto Y$, where $Z$ and $Y$ are not in other literals.
By contrast, we detect semantic redundancy by finding unsatisfiable, implication reducible, recall reducible, and singleton reducible rules.
\citet{reduce} eliminate redundant metarules (second-order rules) to improve the performance of metarule-based ILP approaches \cite{mugg:metagold,hexmil,dai2020abductive}. 
Metarules are templates that define the possible syntax of a hypothesis.
By contrast, we use an ILP system that does not need metarules.

\subsection{Rule Selection}
Many recent systems formulate the ILP problem as a rule selection problem \cite{aspal,ilasp,hexmil,aspsynth}.
These approaches precompute every possible rule in the hypothesis space and then search for a subset that generalises the examples.
Because they precompute all possible rules, they cannot learn rules with many literals.
Moreover, because they precompute all possible rules, they can build pointless rules. 
For instance, if allowed to use the relations \emph{int/1} and \emph{even/1}, \textsc{aspal} \cite{aspal} and \textsc{ilasp} \cite{ilasp} will precompute all possible rules with \emph{int(A)} and \emph{even(A)} in the body.
By contrast, we implement \name{} to work with the constraint-based ILP system \popper{} \cite{popper,maxsynth}.
We bootstrap \popper{} with the constraints discovered by \name{}.
For instance, if \name{} identifies that the pair of literals \emph{int(A)} and \emph{even(A)} are implication reducible, the constraints prohibit \popper{} from building a rule with both those literals. 

\subsection{Constraints}
Many systems use constraints to restrict the hypothesis space \cite{aspal,inoue:mla,atom,ilasp,inspire,hexmil,popper,joiner}.
For instance, the Apperception engine \cite{apperception} has several built-in constraints, such as a \emph{unity condition}, which requires that objects are connected via chains of binary relations.
By contrast, we automatically discover constraints before searching for a hypothesis.
Users can provide constraints to many systems.
For instance, if a user knows that two literals are unsatisfiable, they could provide this information to \ilasp{} in the form of \emph{meta-constraints} or \ale{} via user-defined pruning conditions.
However, in these cases, a user gives the constraints as input to the system.
By contrast, our approach automatically discovers such constraints without user intervention.

\subsection{Preprocessing}
\label{sec:preprocessing}
A key feature of our approach is that we shrink the hypothesis space before giving it an ILP system to search for a hypothesis.
Our approach can, therefore, be seen as a preprocessing step, which has been widely studied in AI, notably to reduce the size of a SAT instance \cite{DBLP:conf/sat/EenB05}.
There is other work on preprocessing in ILP, notably to infer types from the BK \cite{modelearning,DBLP:conf/sigmod/PicadoTFP17}.
For instance, \citet{modelearning} automatically deduce mode declarations from the BK, such as types and whether arguments should be ground.
We do not infer types but remove unsatisfiable, implication reducible, recall reducible, and singleton reducible rules from the hypothesis space.
Moreover, because we use constraints, we can reason about properties that modes cannot capture, such as antitransitivity, mutual exclusivity, and implications.
\citet{bridewell2007} learn structural constraints over the hypothesis space in a multi-task setting. 
By contrast, we discover biases before solving any task.
Other preprocessing approaches in ILP focus on reducing the size of BK \cite{alps,knorf} or predicate invention \cite{playgol,celine:bottom}.
By contrast, we discover constraints in the BK to shrink the hypothesis space.
As a preprocessing step, \citet{queryoptimisation} calculate the recall of the given background relations to order the literals in a rule to reduce the time it takes to check a rule against the examples and BK.
We use a similar algorithm, but rather than order literals in a rule, we use recall to define recall reducible rules (Definition \ref{def:recall_redundant}) and build constraints from them to shrink the hypothesis space.


\subsection{Redundancy in AI}
\citet{plotkin:thesis} uses subsumption to decide whether a literal is redundant in a first-order clause. 
\citet{DBLP:journals/jacm/Joyner76} investigates the same problem, which he calls \emph{clause condensation}, where a condensation of a clause $C$ is a minimum cardinality subset $C'$ of $C$ such that $C' \models C$. 
\citet{gottlob1993removing} improve Joyner's algorithm and show that determining whether a clause is condensed is coNP-complete. 
Detecting and eliminating redundancy is useful in many areas of computer science and has received much attention from the theorem-proving community \cite{HoderKKV12,KhasidashviliK16,VukmirovicBH23}. 
In the SAT community, redundancy elimination techniques, such as blocked clause elimination \cite{Kullmann99}, play an integral role in modern solvers. 
Similar to the goal of our paper, the goal of these redundancy elimination methods is to soundly identify derivationally irrelevant input.

\section{Problem Setting}

We now define the notation used in this paper, provide a short introduction to ILP, define the \emph{hypothesis space reduction problem}, and then define four types of pointless rules that we can remove from the hypothesis space without removing optimal hypotheses.

\subsection{Preliminaries}

We assume familiarity with logic programming \cite{lloyd:book} and ASP \cite{asp} but we restate some key definitions.
An \emph{atom} is of the form $p(t_1,\dots,t_a)$ where $p$ is a predicate symbol of arity $a$ and $t_1,\dots,t_a$ are terms. 
We refer to a ground term as a constant. 
A \textit{literal} is an atom or a negated atom.
A \textit{rule} $r$ is a definite clause of the form $h\leftarrow p_1,\ldots, p_n$ where $h, p_1,\ldots, p_n$ are literals, $head(r) = h$, and $body(r) = \{ p_1,\ldots, p_n\}$. 
We denote the set of variables in a literal $l$ as $vars(l)$.
The variables of a rule $r$, denoted $vars(r)$, is defined as $vars(head(r)) \cup \bigcup_{p\in body(r)} vars(p)$.
Given a rule $r$ and a set of positive literals $C$, by $r\cup C$ we denote the rule $r'$ such that $head(r) = head(r')$ and $body(r') =  body(r)\cup C$.
A substitution $\theta = \{v_1 / t_1, ..., v_n/t_n \}$ is the simultaneous replacement of each variable $v_i$ by its corresponding term $t_i$. 
A rule $r_1$ $\theta$-subsumes a rule $r_2$, denoted $r_1 \preceq_{\theta} r_2$, if and only if there exists a substitution $\theta$ such that $r_1 \theta \subseteq r_2$. 
A rule $r_2$ is a specialisation of a rule $r_1$ if $r_1 \preceq_{\theta} r_2$.

We focus on a restricted form of $\theta$-subsumption \cite{plotkin:thesis} which only considers whether the bodies of two rules with the same head literal are contained in one another. 
We define this restriction through a \emph{sub-rule} relation: 
\begin{definition}[\textbf{Sub-rule}]
\label{def:satruleALT}
Let $r_1$ and $r_2$ be rules,
$head(r_1) = head(r_2)$, and $body(r_1)\subseteq body(r_2)$. 
Then $r_1$ is a \emph{sub-rule} of $r_2$, denoted $r_1\subseteq r_2$.
A rule $r_1$ is a strict sub-rule of $r_2$, denoted $r_1\subset r_2$, if  $head(r_1) = head(r_2)$ and $body(r_1)\subset body(r_2)$.
\end{definition}

\noindent
A \emph{definite program} is a set of definite rules.
A \textit{Datalog program} \cite{datsin:datalog} is a definite program with some key restrictions, such as no function symbols.
We use the term \emph{hypothesis} synonymously with Datalog program.

We generalise the sub-rule relation to a \emph{sub-hypothesis} relation. 
Unlike the sub-rule relation, the sub-hypothesis relation is not a restriction of $\theta$-subsumption:
\begin{definition}[\textbf{Sub-hypothesis}]
\label{def:subhypoth}
Let $h_1$ and $h_2$ be hypotheses and for all $r_1\in h_1$ there exists $r_2\in h_2$ such that $r_1\subseteq r_2$.
Then $h_1$ is a \emph{sub-hypothesis} of $h_2$, denoted $h_1\subseteq h_2$. 
\end{definition}
\noindent
The sub-hypothesis relation captures a particular type of hypothesis we refer to as \emph{basic}. 
These are hypotheses for which specific rules do not occur as part of a recursive predicate definition:
\begin{definition}[\textbf{Basic rule}]
\label{def:basic}
Let $h$ be a hypothesis, 
$r_1$ be a rule in $h$,
and for all $r_2$ in $h$, the predicate symbol of $head(r_1)$ does not occur in a literal in $body(r_2)$.
Then $r_1$ is \emph{basic} in $h$.
\end{definition}
\noindent
As we show below, under certain conditions we can prune hypotheses with respect to their basic rules.  

\subsection{Inductive Logic Programming}
We formulate our approach in the ILP learning from entailment setting \cite{luc:book}.
We define an ILP input:

\begin{definition}[\textbf{ILP input}]
\label{def:probin}
An ILP input is a tuple $(E, B, \mathcal{H})$ where $E=(E^+,E^-)$ is a pair of sets of ground atoms denoting positive ($E^+$) and negative ($E^-$) examples, $B$ is background knowledge, and $\mathcal{H}$ is a hypothesis space, i.e a set of possible hypotheses.
\end{definition}
\noindent
We restrict background knowledge to Datalog programs. 
In the rest of this paper, we assume that the head literal of a rule in a hypothesis does not appear in the head of a rule in the BK.

We define a cost function:
\begin{definition}[\textbf{Cost function}]
\label{def:cost_function}
Given an ILP input $(E, B, \mathcal{H})$, a cost function $cost_{E,B} : \mathcal{H} \to S$ 
maps hypotheses to a totally ordered set $S$.
\end{definition}

\noindent
We use a cost function where $S = \mathbb{N}^2$ with a lexicographic order that first minimises the number of misclassified training examples and then minimises the number of literals in a hypothesis.
Given a hypothesis $h$, a true positive is a positive example entailed by $h \cup B$. 
A true negative is a negative example not entailed by $h \cup B$. 
A false positive is a negative example entailed by $h \cup B$. 
A false negative is a positive example not entailed by $h \cup B$. 
We denote the number of false positives and false negatives of $h$ as $fp_{E,B}(h)$ and $fn_{E,B}(h)$ respectively. 
We consider a function $size : \mathcal{H} \rightarrow {\mathbb{N}}$, which evaluates the size of a hypothesis $h \in \mathcal{H}$ as the number of literals in it. 
In the rest of this paper, we use the cost function:
\begin{align*}
cost_{E,B}(h) = (fp_{E,B}(h) + fn_{E,B}(h), size(h))
\end{align*}

\noindent
 Given an ILP input and a cost function $cost_{E,B}$, we define an \emph{optimal} hypothesis:
\begin{definition}[\textbf{Optimal hypothesis}]
\label{def:opthyp}
Given an ILP input $(E, B, \mathcal{H})$ and a cost function \emph{cost$_{E,B}$}, a hypothesis $h \in \mathcal{H}$ is \emph{optimal} with respect to \emph{cost$_{E,B}$} when $\forall h' \in \mathcal{H}$, \emph{cost$_{E,B}$}($h$) $\leq$ \emph{cost$_{E,B}$}($h'$).
\end{definition}

\subsection{Hypothesis Space Reduction}

Our goal is to reduce the hypothesis space without removing optimal hypotheses:
\begin{definition}[\textbf{Hypothesis space reduction problem}]
\label{def:hypredprob2}
Given an ILP input $(E, B, \mathcal{H})$, the \emph{hypothesis space reduction problem} is to find $\mathcal{H}' \subset \mathcal{H}$ such that if $h \in \mathcal{H}$ is an optimal hypothesis then $h \in \mathcal{H}'$.
\end{definition}

\noindent 
In this paper, we shrink the hypothesis space by using the BK to find rules that cannot be in an optimal hypothesis.
We consider four types of rules: \emph{unsatisfiable}, \emph{implication reducible}, \emph{recall reducible}, and \emph{singleton reducible}.
We describe these in turn.

\subsection{Unsatisfiable Rules}
An unsatisfiable rule has a body that can never be true given the BK.
For instance, consider the rule: 
\begin{center}
\begin{tabular}{l}
\emph{h $\leftarrow$ even(A), odd(A), int(A)}
\end{tabular}
\end{center}
\noindent
Assuming that the relations \emph{odd} and \emph{even} are mutually exclusive, this rule is unsatisfiable because its body is unsatisfiable.

As a second example, consider the rule: 
\begin{center}
\begin{tabular}{l}
\emph{h $\leftarrow$ succ(A,B), succ(B,C), succ(A,C)}\\
\end{tabular}
\end{center}
\noindent
Assuming a standard definition for the \emph{succ/2} relation, this rule is unsatisfiable because \emph{succ/2} is \emph{antitransitive}.

We define an unsatisfiable rule:

\begin{definition}
[\textbf{Unsatisfiable rule}]
\label{def:unsatrule}
Let $B$ be BK and $r$ be a rule with the body $b$.
Then $r$ is unsatisfiable if there is no grounding substitution $\theta$ such that $B \models b\theta$.
\end{definition}



\noindent


\noindent
We show that a hypothesis with an unsatisfiable rule cannot be optimal.

\begin{proposition}[\textbf{Unsatisfiability soundness}] 

\label{prop:sound_unsat2}
Let $B$ be BK, $h$ be a hypothesis, 
$r \in h$, 
$r' \subseteq r$,
and $r'$ be unsatisfiable with respect to $B$.
Then $h$ is not optimal. 
\end{proposition}


\begin{proof}
Assume the opposite, \emph{i.e.} that $h$ is optimal.
Let $h' = h \setminus \{r\}$.
Since $r' \subseteq r$ then $r'$ $\theta$-subsumes $r$ and thus $r'\models r$.
Since $r'$ is unsatisfiable w.r.t. $B$ then $r$ is unsatisfiable w.r.t. $B$.
Because $r$ is unsatisfiable, it does not influence the set of derivable facts of $h$, so $h$ and $h'$ entail exactly the same examples, i.e. $fp_{E,B}(h) = fp_{E,B}(h')$ and $fn_{E,B}(h) = fn_{E,B}(h')$.
Since $h' = h \setminus \{r\}$ then $size(h') < size(h)$.
Therefore, 
$cost_{E,B}(h') < cost_{E,B}(h)$, so $h$ cannot be optimal, contradicting the assumption.
\end{proof}

\subsection{Implication Reducible Rules}

An \emph{implication reducible} rule has a body literal that is implied by other body literals.
For example, consider the rules:
\begin{center}
\begin{tabular}{l}
\emph{r$_1$  = h $\leftarrow$ odd(A), int(A)}\\
\emph{r$_2$  = h $\leftarrow$ odd(A)}
\end{tabular}
\end{center}
\noindent
Assuming that \emph{odd} implies \emph{int},
the rule \emph{r$_1$} is implication reducible because its body contains a redundant literal, \emph{int(A)}.
Therefore, \emph{r$_1$} is logically equivalent to \emph{r$_2$} (with respect to BK).

As a second example, consider the rule: 
\begin{center}
\begin{tabular}{l}
\emph{h $\leftarrow$ gt(A,B), gt(B,C), gt(A,C)}\\
\end{tabular}
\end{center}
\noindent
Assuming a standard definition for the \emph{gt/2} relation, this rule is implication reducible because \emph{gt/2} is transitive, i.e. \emph{gt(A,B), gt(B,C)} implies \emph{gt(A,C)}.

Because an implication reducible rule contains a redundant literal, an implication reducible rule cannot be in an optimal hypothesis.
However, a specialisation of an implication reducible rule can be in an optimal hypothesis.
For instance, consider the rule:
\begin{center}
\begin{tabular}{l}
\emph{$r_1$ = h $\leftarrow$ member(X,L), member(Y,L)}\\
\end{tabular}
\end{center}
\noindent
In this rule, \emph{member(X,L)} implies \emph{member(Y,L)} and vice-versa, so one of the literals is redundant.
However, we could still specialise this rule as:
\begin{center}
\begin{tabular}{l}
\emph{$r_2$ = h $\leftarrow$ member(X,L), member(Y,L), gt(Y,X)}\\
\end{tabular}
\end{center}

\noindent
The rule $r_1$ is not logically equivalent to $r_2$ and $r_2$ may appear in an optimal hypothesis.

We identify implication reducible rules where we can prune all its specialisations.
The idea is to identify a redundant \emph{captured} literal, which is a literal implied by other literals and where all of its variables appear elsewhere in the rule.
For instance, consider the rule:

\begin{center}
\begin{tabular}{l}
\emph{h $\leftarrow$ succ(A,B), succ(B,C), gt(C,A), gt(C,D)}\\
\end{tabular}
\end{center}
\noindent
In this rule, the literal \emph{gt(C,A)} is captured because its variables all appear elsewhere in the rule.
By contrast, the literal \emph{gt(C,D)} is not captured because the variable \emph{D} does not appear elsewhere in the rule.

We define a captured literal:
\begin{definition}[\textbf{Captured literal}]
\label{def:capturedLiteral}
Let 
$r$ be a rule,
$l \in body(r)$, 
and
$vars(l) \subseteq vars( body(r) \setminus \{l\})$. 
Then $l$ is $r$-\emph{captured}. 
\end{definition}

\noindent
If a literal is captured in a rule then it is captured in its specialisations:

\begin{lemma}
\label{prop:capTrans}
Let $r_1$ be a rule,
$r_2\subseteq r_1$,
$l\in body(r_2)$,
and $l$ be $r_2$-captured.
Then $l$ is $r_1$-captured.
\end{lemma}
\begin{proof}
Follows from Definition~\ref{def:satruleALT} as the sub-rule relation preserves variable occurrence. 
\end{proof}
\noindent
We define an implication reducible rule:

\begin{definition}[\textbf{Implication reducible}]
\label{def:breducible}
Let $r$ be a rule, 
$B$ be BK,
$l \in body(r)$ be $r$-captured,
and $B\models r\leftrightarrow r \setminus \{l\}$.
Then  $r$ is \emph{implication reducible}.
\end{definition}

\noindent Some specialisations of an implication reducible rule are implication reducible:

\begin{proposition}[\textbf{Implication reducible specialisations}] \label{prop:specRedu}
Let $B$ be BK, 
$r_1$ be an implication reducible rule, 
and $r_1\subseteq r_2$. 
Then $r_2$ is implication reducible.
\end{proposition}


\begin{proof}
Since $r_1$ is implication reducible then, by definition 10, there exists $l \in body(r_1)$ s.t. $l$ is $r_1$-captured and $B\models r_1\leftrightarrow (r_1 \setminus \{l\})$.
Since $l$ is $r_1$-captured then, by Lemma~\ref{prop:capTrans}, it is also $r_2$-captured.
Because $l$ is $r_2-captured$, every variable in $l$ occurs elsewhere in $r_2$ so removing $l$ does not remove any variables from $r_2$.
Let $C = body(r_2) \setminus body(r_1)$.
Then 
$B\models (r_1\cup C)\leftrightarrow ((r_1\cup C) \setminus \{l\})$.
Finally, since $r_2 = (r_1\cup C)$ then $B\models r_2\leftrightarrow (r_2 \setminus \{l\})$, so $r_2$ is implication reducible.
\end{proof}



\noindent 
Certain hypotheses that contain a sub-hypothesis with implication reducible rules are not optimal:
\begin{proposition}[\textbf{Implication reducible soundness}] \label{prop:sound_sat3}
Let $B$ be BK,
$h_1$ be a hypothesis, 
$h_2\subseteq h_1$, 
$r_1$ be a basic rule in $h_1$, 
$r_2\in h_2$,
$r_2\subseteq r_1$,
and $r_2$ be implication reducible with respect to $B$.
Then $h_1$ is not optimal. 
\end{proposition}
\begin{proof} By Proposition~\ref{prop:specRedu}, $r_1$ is also reducible implying that there exists a rule $r_3\subseteq r_1$ such that (i)  $B\models r_1 \leftrightarrow r_3$ and (ii) $|r_3|< |r_1|$. 
Let $h_3 = (h_1\setminus \{r_1\})\cup\{r_3\}$. Then $cost_{E,B}(h_3) < cost_{E,B}(h_1)$, i.e. $h_1$ is not optimal.
\end{proof}

\subsection{Recall Reducible Rules}

A \emph{recall reducible} rule contains a redundant literal 
determined by the number of ground instances of a literal deducible from the BK \cite{ilpintro}.
For example, consider the rule:

\begin{center}
\begin{tabular}{l}
\emph{h $\leftarrow$ succ(A,B), succ(A,C)}\\
\end{tabular}
\end{center}

\noindent
This rule contains a redundant literal 
(\emph{succ(A,B)} or \emph{succ(A,C)}) because the second argument of \emph{succ/2} is functionally dependent on the first argument.
In other words, if \emph{succ(A,B)} and \emph{succ(A,C)} are true then \emph{B=C}.

As a second example, consider the rule:
\begin{center}
\begin{tabular}{l}
\emph{h $\leftarrow$ add(A,B,C), add(A,B,D)}
\end{tabular}
\end{center}
This rule contains a redundant literal 
(\emph{add(A,B,C)} or \emph{add(A,B,D)}) because the third argument of \emph{add/3} is functionally dependent on the first two arguments.
In other words, if \emph{add(A,B,C)} and \emph{add(A,B,D)} are true then \emph{C=D}.


Our implication reducible definition (Definition \ref{def:breducible}) requires literals to be captured, so is insufficient for functional dependencies.
For instance, in the above \emph{succ/2} example, neither literal is captured. 
We therefore introduce the notion of a \emph{recall reducible rule} to capture this type of redundancy.
Specifically, our definition captures the observation that redundancy occurs when multiple literals with the same predicate are instantiated with identical constants. 
In these cases, there exists another rule that contains fewer literals in the body and logically implies the original rule.
Our formal definition is:


\begin{definition}[\textbf{Recall reducible}]
\label{def:recall_redundant}
Let 
$B$ be BK,
$r_1$ be a rule,
$r_1\preceq_{\theta} r_2$,
$|r_1|>|r_2|$, 
and $B\models r_1\leftrightarrow r_2$.
Then $r_1$ is recall reducible.
\end{definition}

\noindent
We show how recall reducible applies to a rule:
\begin{example}
The following example illustrates that $r_1$ (as defined below) is recall reducible. Observe that $r_1\preceq_{\theta} r_2$ as $r_1\theta = r_2$, $|r_1|>|r_2|$, and $B\models r_1\rightarrow r_2$ follows from $r_1\preceq_{\theta} r_2$.  Furthermore, $B\models r_2\rightarrow r_1$ follows from \textit{succ/2} being injective.
\hspace{1em}

\begin{center}
\begin{tabular}{l}
\emph{$B$ = \{succ(1,2), succ(2,3), \dots \}}\\
\emph{$r_1$ = h $\leftarrow$ succ(X,Y), succ(X,Z)}\\
\emph{$r_2$ = h $\leftarrow$ succ(U,V)}\\
\emph{$\theta$ = $\{X \mapsto U, Y\mapsto V, Z \mapsto V\}$} \\
\end{tabular}
\end{center}
\end{example}

\begin{example}
The following example illustrates that $r_1$ (as defined below) is recall reducible. Observe that $r_1\preceq_{\theta} r_2$ as $r_1\theta = r_2$, $|r_1|>|r_2|$, and $B\models r_1\rightarrow r_2$ follows from $r_1\preceq_{\theta} r_2$.  Furthermore, $B\models r_2\rightarrow r_1$ follows from \textit{edge/2} being a bijective mapping from $\{a,b,c\}$ to itself.

\begin{center}
\begin{tabular}{l}
\emph{$B$ = \{edge(a,b), edge(b,c), edge(c,a) \}}\\
\emph{$r_1$ = h $\leftarrow$ edge(A,B), edge(B,C), edge(C,D), edge(D,E)}\\
\emph{$r_2$ = h $\leftarrow$ edge(A,B), edge(B,C), edge(C,A)}\\
\emph{$\theta$ = $\{D \mapsto A, E \mapsto B\}$} \\

\end{tabular}
\end{center}
\end{example}



\noindent
We show that a specialisation of a recall reducible rule is also recall reducible:

\begin{proposition}[\textbf{Recall specialisation}] \label{prop:specRedundan}
Let $B$ be BK, 
$r_1$ be a recall reducible rule, 
and $r_1\subseteq r_2$. 
Then $r_2$ is recall reducible.
\end{proposition}
\begin{proof}[Proof (Sketch)]
See Appendix~B for the full proof. 
Essentially, $r_1$ contains more literals with the same predicate symbol than the background knowledge supports with distinct values, i.e. some literals must be redundant. If we add additional literals to the body of $r_1$, regardless of their predicate symbol, the resulting rule, $r_2$ will still contain the redundancies of $r_1$. 
\end{proof}

\noindent
A hypothesis with a recall reducible rule cannot be optimal:
\begin{proposition}[\textbf{Recall reducible soundness}] \label{prop:sound_sat4}
Let $B$ be BK,
$h_1$ be a hypothesis, 
$h_2\subseteq h_1$, 
$r_1$ be a basic rule in $h_1$, 
$r_2\in h_2$,
$r_2\subseteq r_1$,
and $r_2$ be recall reducible with respect to $B$.
Then $h_1$ is not optimal. 
\end{proposition}
\begin{proof} By Proposition~\ref{prop:specRedundan}, $r_1$ is  recall reducible as $r_2\subseteq r_1$. Thus, there exists a substitution $\theta$, a rule $r_3 = r_1\theta$, and hypothesis $h_3 = (h_1\setminus \{r_1\})\cup\{r_3\}$ such that (i) $B\models r_1\leftrightarrow r_3$ and (ii) $|r_3|< |r_1|$.  It follows that $cost_{E,B}(h_3) < cost_{E,B}(h_1)$, i.e. $h_1$ is not optimal.
\end{proof}

\noindent
Appendix C details the difference between implication and recall reducible rules.

\subsection{Singleton Reducible Rules}

A \emph{singleton reducible} rule contains a literal that is always true.
For instance, consider the rule:

\begin{center}
\begin{tabular}{l}
\emph{f(A) $\leftarrow$ length(A,B)} 
\end{tabular}
\end{center}
\noindent

\noindent
If $A$ is of type \texttt{list} then \emph{length(A,B)} is always true because every list has a length.
In other words, because the variable $A$ of length is total and the variable $B$ only appears once in this rule, the literal \emph{length(A,B)} is redundant because it has no discriminatory power. 

We can identify total literals by separating its variables into two sets $S_1$ and $S_2$ and observing that for any instantiation of variables in $S_1$ there exists an instantiation of the variables in $S_2$ such that the literal is true. 
For instance, consider the rule 
\emph{divides(A,B,C) $\leftarrow$ A/B = C}. 
Separate the variables in the sets $S_1=\{B\}$ and $S_2=\{A,C\}$. 
If we instantiate $B$ with $0$ then \emph{(A,0,C)} is false for any  instantiation of $A$ and $C$. Thus, we say that \emph{divides(A,B,C)} is partial in $S_1$. 
If instead we chose $S_1=\{C\}$ and $S_2=\{A,B\}$, then  for every instantiation of $C$ there is an instantiation of $A$ and $B$ such that \emph{divides(A,B,C)} is true.  
Thus, we say that \emph{divides(A,B,C)} is total in $S_1$. 

In this section, we assume a simple typing framework, which we describe in Appendix A, and that all literals are well typed.
We first define a literal instance:

\begin{definition}[\textbf{Instance}]
Let 
$l$ be a well-typed literal, 
$S\subseteq vars(l)$, 
and $\theta$ be a well-typed substitution with respect to $l$ that maps each variable in $S$ to a constant.
Then $l\theta$ is an instance of $l$ over $S$.  
The set of $S$ instances of $l$ is denoted $inst(S,l)$.
\end{definition}

\noindent
We define a \emph{total literal}:

\begin{definition}[\textbf{Total literal}]
\label{def:partial}
Let 
$B$ be BK, 
$l$ a literal, and $S\subseteq vars(l)$. 
Then $l$ is \emph{total} in $S$ if for every $l'\in inst(S,l)$, $B\models l'$; otherwise, $l$ is partial in $S$.
\end{definition}

\noindent
In addition to checking whether a given set $S$ of variable argument positions is total in a given literal $l$, we need to check whether the positions which are not in $S$ only occur in $l$; otherwise, the removal of the literal could change the semantics of the rule containing $l$.   
For instance, consider the rule \emph{p(A,B) $\leftarrow$ add(A,B,C), mul(A,B,C)}.
Observe that \emph{p(A,B)} is true when $A=B=0$ or $A=B=2$. 
Furthermore, no matter how we choose $S$, \emph{add(A,B,C)} is total in $S$. 
If we remove \emph{add(A,B,C)} from the above rule, the resulting rule accepts any pair of numbers. 

As this example illustrates, we need to take care when removing total literals. 
We define the precise requirements:

\begin{definition}[\textbf{Singleton reducible literal}]
\label{def:singleton}
Let 
$r$ be a rule,
$l$ be a body literal of $r$,
$S\subseteq vars(l)$ such that $l$ is total in $S$ and each  $V \in vars(l)\setminus S$ occurs precisely once in $r$. Then $l$ is \emph{singleton reducible}. By $tot(l)$  we denote $vars(l)\setminus S$.
\end{definition}

\noindent
If a literal is singleton reducible then it does not affect whether the rule is entailed by the BK:

\begin{proposition}
\label{prop:6}
Let 
$B$ be BK, 
$r$ be a rule, 
and $l$ be a singleton reducible literal in $r$. 
Then $B\models r\leftrightarrow r\setminus\{l\}$.
\end{proposition}

\begin{proof}
Let $r = (h \leftarrow b,l)$ and $r'=(h \leftarrow b)$.
For $B\models r'\rightarrow r$, it follows 
since $r'$ subsumes $r$.
For $B\models r \rightarrow r'$, let $\sigma$ be any grounding substitution for the variables of $r'$ and assume $B\models b\sigma$. 
Since $l$ is singleton reducible in $r$, there exists $S\subseteq vars(l)$ such that
$l$ is total in $S$ and each variable in $vars(l)\setminus S$ occurs precisely once in $r$ and thus only in $l$.
Extend $\sigma$ to a grounding substitution $\theta$ for all variables of $r$ such that $B\models l\theta$, where the extra variables only occur in $l$. Such a substitution is constructable by Definition~\ref{def:partial} as $l$ is total in $S$.
Thus $B\models (b\wedge l)\theta$, and using $r$ we obtain $B\models h\theta$, hence $B\models h\sigma$.
Therefore $B\models r\to r'$.
Hence $B\models r\leftrightarrow r'$.
\end{proof}

\noindent
A \emph{singleton reducible rule} is a rule with a singleton reducible literal.
The specialisations of a singleton reducible rule are also singleton reducible if the reducible literal remains total:

\begin{proposition}[\textbf{Singleton specialisation}] \label{prop:specSingleton}
Let $B$ be BK, 
$r_1$ be a rule,
$l$ be a singleton reducible literal in $r_1$,
$r_1\subseteq r_2$, 
and variables in $tot(l)$ are singleton in $r_2$. 
Then $r_2$ is a singleton reducible rule. 
\end{proposition}
\begin{proof}
Since $r_1 \subseteq r_2$ then $l \in body(r_2)$.
The totality of $l$ only depends on $B$, so $l$ is total in $r_2$.
By assumption, the variables in $tot(l)$ are singleton in $r_2$. 
Therefore, by Definition \ref{def:singleton}, $l$ is singleton reducible in $r_2$.
\end{proof}

\noindent 
We now show a soundness result similar to those introduced in previous sections:
\begin{proposition}[\textbf{Singleton reducible soundness}] \label{prop:sound_sat5}
Let $B$ be BK,
$h_1$ be a hypothesis, 
$h_2\subseteq h_1$, 
$r_1$ be a basic rule in $h_1$, 
$r_2\in h_2$,
$r_2\subseteq r_1$,
 $r_2$ be singleton reducible with respect to $B$ and a literal $l\in body(r)$, and variables in $tot(l)$ are singleton in $r_1$.
Then $h_1$ is not optimal. 
\end{proposition}

\begin{proof} 
By Proposition~\ref{prop:specSingleton}, $r_1$ is also singleton reducible. 
Let $h_3 = (h_1\setminus \{r_1\})\cup\{r_1\setminus\{l\}\}$.  
By Proposition \ref{prop:6}, $B\models r \rightarrow r\setminus\{l\} $.
Since $r_1$ is basic in $h_1$, it holds that for all $r'\in h_1\setminus {r_1}$, $head(r_1) \not \in body(r')$. 
Hence, $B\models (h_1\setminus \{r_1\}) \leftrightarrow  (h_3\setminus \{r_1\setminus \{l\}\})$. 
Furthermore, the singleton reducibility of $r_1$ implies that 
 $h_1$ and $h_3$ entail exactly the same examples, i.e. $fp_{E,B}(h_1) = fp_{E,B}(h_3)$ and $fn_{E,B}(h_1) = fn_{E,B}(h_3)$.
But $cost_{E,B}(h_3) < cost_{E,B}(h_1)$.
Therefore, $h_1$ is not optimal.
\end{proof}
\section{\name{}}
\label{sec:hs-reduction}
The previous section outlines four types of pointless rules that we can deduce from BK.
We now describe \name{}, which automatically finds such pointless rules to remove from the hypothesis space.
Algorithm \ref{alg:shrinker} shows the high-level algorithm.
This algorithm takes as input BK in the form of a Datalog program.
In the rest of this section, we assume, for simplicity, that the BK has been materialised into a set of facts \cite{materialisation}.

\begin{algorithm}[ht!]
\large
{
\begin{myalgorithm}[]
def shrinker(bk, max_size, max_vars, batch_size, timeout):
  pointless_rules = {}
  pointless_rules += find_unsat_impli(bk, max_size, max_vars, batch_size, timeout)
  pointless_rules += find_recall_reducible(bk)
  pointless_rules += find_singleton_reducible(bk)
  return pointless_rules
\end{myalgorithm}
\caption{
\name{}
}
\label{alg:shrinker}
}
\end{algorithm}

\name{} has three components.
The first finds unsatisfiable and implication reducible rules 
(line 3).
The second finds recall reducible rules 
(line 4).
The third finds singleton reducible rules 
(line 5).
We describe these components in turn.
We then describe how \name{} uses the pointless rules to shrink the hypothesis space of a constraint-based ILP system.
The output of the algorithm is independent of the order of the three steps.

\subsection{Finding Unsatisfiable and Implication Reducible Rules}
\label{sec:unsat_implication_reducible}

To find unsatisfiable and implication reducible rules, \name{} builds small rule templates and uses the BK to check whether any instances are unsatisfiable or implication reducible.
A \emph{template} is a set of second-order literals.
For instance, the set \emph{\{P(A,B), Q(B,A)\}} is a template where \emph{P} and \emph{Q} are second-order variables and \emph{A} and \emph{B} are first-order variables.
An \emph{instance} of a template is a grounding of the second-order variables, such as \emph{\{succ(A,B), succ(B,A)\}}.
\name{} considers batches of templates.
As there can be many possible templates, \name{} uses a timeout to find as many unsatisfiable or implication reducible rules as possible in the time limit.

Algorithm~\ref{alg:shrinker1} shows the algorithm for finding unsatisfiable and implication reducible rules.
The algorithm takes the inputs: 
background knowledge (\emph{BK}),
a maximum number of literals in a rule (\emph{max\_size}),
a maximum number of unique variables in a rule (\emph{max\_vars}),
the number of templates to check in each iteration (\emph{batch\_size}),
and a time limit (\emph{timeout}).
The algorithm returns a set of pointless rules.

\begin{algorithm}[ht!]
\large
{
\begin{myalgorithm}[]
def find_unsat_impli(bk, max_size, max_vars, batch_size, timeout):
  templates = build_templates(bk, max_size, max_vars)    
  start_time = get_time()
  pointless_rules = {}
  while (get_time() - start_time) < timeout:
    check_templates = select_templates(templates, batch_size)
    templates -= check_templates
    pointless_rules += find_pointless_rules(bk, check_templates)
  return pointless_rules
\end{myalgorithm}
\caption{
Find unsatisfiable and reducible rules.
}
\label{alg:shrinker1}
}
\end{algorithm}

Algorithm~\ref{alg:shrinker1} works as follows.
The function \emph{build\_templates} (line 2) builds all possible templates with at most \emph{max\_size} literals and at most \emph{max\_vars} variables and returns them in ascending order of size (number of literals).
Lines 5-8 show the loop to test templates on the BK to find unsatisfiable and implication reducible rules.
The function \emph{select\_templates} (line 6) selects \emph{batch\_size} untested templates to test.
The function \emph{find\_pointless\_rules} (line 8) is  key.
It takes as input BK and templates and returns the discovered pointless rules.

We implement the \emph{find\_pointless\_rules} function in a bottom-up way using ASP.
Specifically, we search for a counter-example for each rule to show that it is not unsatisfiable or not implication reducible.
For instance, to show that any rule with \emph{p(A,B), p(B,C), p(A,C)} in the body is unsatisfiable, we need to check there is no counter-example where the literals \emph{p(A,C), p(B,C), p(A,C)} all simultaneously hold in the BK.
Likewise, to show that any rule with the literals \emph{p(A), q(A)} in the body is implication reducible, we need to check there is no counter-example where either \emph{p(A)} holds and \emph{q(A)} does not, or vice-versa.

We use ASP to perform these checks.
We first build an ASP encoding $\mathcal{P}$ which includes all the BK.
In the following, we use \texttt{typewriter typeface} to denote ASP code.
For each relation \emph{p} with arity $a$ in the BK, we add these rules to $\mathcal{P}$:

\begin{lstlisting}[basicstyle=\ttfamily, xleftmargin=2em,   mathescape=true]
pred(p,a).
holds(p,($X_1,X_2,\dots,X_a$)):- p($X_1,X_2,\dots,X_a$).
\end{lstlisting}

\noindent
We use the given templates to build ASP rules to add to $\mathcal{P}$ to find pointless rules.
We describe each type of rule in turn.

\subsubsection{Unsatisfiable Rules}

Given the template \emph{\{P(A,B), Q(B,A)\}}, we add these ASP rules to $\mathcal{P}$:

\begin{lstlisting}[basicstyle=\ttfamily, xleftmargin=2em]   
sat_ab_ba(P,Q):- holds(P,(A,B)), holds(Q,(B,A)).
unsat_ab_ba(P,Q):- pred(P,2), pred(Q,2), not sat_ab_ba(P,Q).
\end{lstlisting}


\noindent
The atom \emph{sat\_ab\_ba(P,Q)} is true for the predicate symbols \emph{P} and \emph{Q} if the literals \emph{P(A,B)} and \emph{Q(B,A)} are both true for any values \emph{A} and \emph{B}.
This atom can be seen as a counter-example to say that this template with these predicate symbols is satisfiable.
The atom \emph{unsat\_ab\_ba(P,Q)} is true for the predicate symbols \emph{P} and \emph{Q} if there is no satisfiable counter-example.
These ASP rules allow us to deduce asymmetry for binary relations.

Given the template \emph{\{P(A,B), Q(B,C), R(A,C)\}}, we add these ASP rules to $\mathcal{P}$:
\begin{lstlisting}[basicstyle=\ttfamily, xleftmargin=2em]   
sat_ab_bc_ac(P,Q,R):- holds(P,(A,B)), holds(Q,(B,C)), holds(R,(A,C)).
unsat_ab_bc_ac(P,Q,R):- pred(P,2), pred(Q,2), pred(R,2), not sat_ab_bc_ac(P,Q,R).
\end{lstlisting}

\noindent
These rules allow us deduce antitransitivity for binary relations.

\subsubsection{Implication Reducible Rules}

Given the template \emph{\{P(A), Q(A)\}}, we add these ASP rules to $\mathcal{P}$:

\begin{lstlisting}[basicstyle=\ttfamily, xleftmargin=2em]   
aux_a_a(P,Q):- holds(P,(A,)),  not holds(Q,(A,)).
implies_a_a(P,Q):- pred(P,1), pred(Q,1), P != Q, not aux_a_a(P,Q).    
\end{lstlisting}

\noindent
The atom \emph{aux\_a\_a(P,Q)} is true for the predicate symbols \emph{P} and \emph{Q} if there is any literal \emph{P(A)} that is true where the literal \emph{Q(A)} is not true.
This atom can be seen as a counter-example to say that \emph{P(A)} does not imply \emph{Q(A)}.
The atom \emph{implies\_a\_a(P,Q)} is true for the predicate symbols \emph{P} and \emph{Q} if there is no implication counter-example.

Given the template \emph{\{P(A,B), Q(B,C), R(A,C)\}}, we add these ASP rules to $\mathcal{P}$:

\begin{lstlisting}[basicstyle=\ttfamily, xleftmargin=2em]   
aux_ab_bc_ac(P,Q,R):-  holds(P,(A,B)),  holds(Q,(B,C)),  not holds(R,(A,C)).
implies_ab_bc_ac(P,Q,R):- pred(P,2), pred(Q,2), pred(R,2), not aux_ab_bc_ac(P,Q,R).    
\end{lstlisting}

\noindent
\subsubsection{Deducing Unsatisfiable and Implication Reducible Rules}

We use the ASP solver \textsc{clingo} \cite{multishot-clingo} to find a model of $\mathcal{P}$ and thus find unsatisfiable and implication reducible rules.
Specifically, we use the multi-shot solving feature of \textsc{clingo} \cite{multishot-clingo} to test batches of templates incrementally without regrounding BK.
\name{} represents a pointless rule as an atom.
For instance, if the template \emph{\{succ(A,B), succ(B,A)\}} is unsatisfiable then \name{} represents it as the atom \emph{unsat\_ab\_ba(succ,succ)}. 
Algorithm \ref{alg:shrinker1} returns a set of such atoms, denoting pointless rules.

\subsection{Finding Recall Reducible Rules}
\label{sec:recall_reducible}

To find recall reducible rules, we use an approach inspired by  \citet{queryoptimisation} for ordering literals in a rule to improve query efficiency.
See Section \ref{sec:preprocessing} for more detail on this work.
At a high level, for any background relation, we determine the maximum number of answer substitutions (ground instantiations) for any subset of its arguments.
We use standard recall notation \cite{progol} to denote input/ground \emph{(+)} and output/non-ground \emph{(-)} arguments.
To illustrate this notation, consider the facts:
\begin{center}
\begin{lstlisting}[basicstyle=\ttfamily, frame=single]
p(1,2). p(2,1). p(3,1). 
q(p1,a,b). q(p2,b,c). q(p3,a,b). q(p4,b,c). 
\end{lstlisting}
\end{center}
\noindent
The recall for \emph{p(-,-)} is 3 because there are only 3 \emph{p/2} facts.
The recall for \emph{p(+,-)} is 1 because, for any ground first argument, there is at most 1 ground second argument.
The recall for \emph{q(+,-,-)} is also 1 because, for any first argument, there is at most one pair formed of arguments 2 and 3. 
The recall for \emph{q(-,+,+)} is 2 because for any pair of arguments formed of arguments 2 and 3 there are at most 2 non-ground arguments in position 1.

Algorithm \ref{alg:recalls} shows our algorithm for finding recall reducible rules.
The algorithm takes as input BK and returns a set of atoms denoting the recall of a relation with respect to a subset of its arguments.
The algorithm works as follows.
It loops over every atom in the BK (line 4) and every subset of the arguments of an atom (line 9), which corresponds to all recall combinations.
For instance, for the atom \emph{succ(A,B)}, using the recall notation, the subsets are \emph{succ(-,-)}, \emph{succ(+,-)}, and \emph{succ(-,+)}.
We ignore the case when all the arguments are input/ground \emph{(succ(+,+)}) because the recall is one.
For each subset, we create a key formed of the input/ground arguments and a value formed of the output/non-ground arguments.
We then add the keys and values to a hashmap that is specific for the predicate symbol and its argument subset. 
Finally, for each relation and subset of its arguments, we use this hashmap to calculate the maximum answer substitutions (recall).
The algorithm returns a set of atoms of the form \emph{recall\_pred\_input\_output\_count}, where \emph{pred} is a predicate symbol, \emph{input} is a sequence of input variables, \emph{output} is a sequence of output variables, and \emph{count} is the corresponding recall value.
For instance, for \emph{p(+,-)} with recall 1, the atom is \emph{recall\_p\_a\_b\_1}.
For \emph{p(-,+)} with recall 1, the atom is \emph{recall\_p\_b\_a\_1}.
For \emph{q(-,+,+)} with recall 2, the atom is \emph{recall\_q\_bc\_a\_2}.

Algorithm \ref{alg:recalls} is exponential in the maximum 
arity of a predicate symbol in the background knowledge because it iterates over the powerset of argument positions (line 9).
In practice, the maximum arity is small, usually <4.

\begin{algorithm}[ht!]
{
\begin{myalgorithm}[]
def find_recall_redundant(bk):
  counts = {}
  preds = {}
  for atom in bk:
    pred = get_pred(atom)    
    args = get_args(atom)
    arity = get_arity(atom)
    preds += (pred, arity)
    for arg_subset in powerset({1,...,arity}):     
        key = ''
        value = ''
        for i in range(arity):
          if i in arg_subset:
            key += args[i]
          else:
            value += args[i]
        counts[pred][arg_subset][key].add(value)

  recalls = {}
  for pred, arity in preds:
    for arg_subset in powerset({1,...,arity}):  
      keys = counts[pred][arg_subset]
      recall = max(len(counts[pred][arg_subset][key]) for key in keys)
      recalls[(pred, arg_subset)] = recall
  return recalls
    
\end{myalgorithm}
\caption{
Find recall reducible rules
}
\label{alg:recalls}
}
\end{algorithm}

\subsection{Finding Singleton Reducible Rules}
\label{sec:singleton_reducible}

We find singleton reducible rules using ASP.
As a reminder, a singleton reducible rule contains a literal that is always true.
For instance, consider the rule:

\begin{center}
\begin{tabular}{l}
\emph{f(A) $\leftarrow$ length(A,B)} 
\end{tabular}
\end{center}
\noindent

\noindent
Assuming sensible semantics for \emph{length/2}, the literal \emph{length(A,B)} is always true because every list has a length.
In other words, because the variable $B$ only appears once in this rule and $A$ is true for all lists, the literal \emph{length(A,B)} is redundant as it has no discriminatory power.

We assume that the background relations are simply typed (see Appendix A for an elaboration on types).
We use ASP to determine whether a relation is partial or total and use that information to deduce singleton reducible rules.
We describe our ASP encoding $\mathcal{P}$.
We add the BK to $\mathcal{P}$.
For each relation \emph{p} with arity $a$ in the BK, we add this rule to $\mathcal{P}$:

\begin{lstlisting}[basicstyle=\ttfamily, xleftmargin=2em,   mathescape=true]
holds(p,($X_1,X_2,\dots,X_a$)):- p($X_1,X_2,\dots,X_a$).
\end{lstlisting}


\noindent
For each relation $p$ with arity $a$ and each index $i \in \{1, \dots, a\}$ with type $t$ we add this fact to $\mathcal{P}$:

\begin{lstlisting}[basicstyle=\ttfamily, xleftmargin=2em,   mathescape=true]
type(p,i,t).
\end{lstlisting}

\noindent
For each relation $p$ with arity $a$ and for every index $i=1 \dots a$, we determine whether a constant symbol $C_i$ holds at $i$ for $p$ by adding this rule to $\mathcal{P}$:


\begin{lstlisting}[basicstyle=\ttfamily, xleftmargin=2em,  mathescape=true]
cholds(p,i,C$_i$):- holds(p,(C$_1$,C$_i$,$\dots$,C$_a$)). 
\end{lstlisting}

\noindent
We determine the constants that hold for a type $T$ (the domain of $T$) by adding this rule to $\mathcal{P}$:

\begin{lstlisting}[basicstyle=\ttfamily, xleftmargin=2em,  mathescape=true]
domain(T,C):- cholds(P,I,C), type(P,I,T).
\end{lstlisting}

\noindent
For a predicate symbol $p$ with arity $a>1$, we determine whether two constant symbols $C_i$ and $C_j$ hold at indices $i$ and $j$ respectively for $i=1\dots j-1, j\dots a$ by adding this rule to $\mathcal{P}$:

\begin{lstlisting}[basicstyle=\ttfamily, xleftmargin=2em,  mathescape=true]
holds2(p,i,j,C$_i$,C$_j$):- cholds(p,i,C$_i$), cholds(P,j,C$_j$).
\end{lstlisting}

\noindent
We add similar rules for all values $1 \leq i_1 < i_2 < \dots < i_k < a$, such as this rule:

\begin{lstlisting}[basicstyle=\ttfamily, xleftmargin=2em,  mathescape=true]
holds3(p,i,j,k,C$_i$,C$_j$,C$_k$):-  cholds(p,i,C$_i$), cholds(p,j,C$_j$), cholds(p,k,C$_k$).
\end{lstlisting}

\noindent
We deduce whether an argument position $I$ is partial with the rule:

\begin{lstlisting}[basicstyle=\ttfamily, xleftmargin=2em,  mathescape=true]
partial(P,I):- type(P,I,T), domain(T,X), not cholds(P,I,X).
\end{lstlisting}

\noindent
We add similar rules for multiple arguments.
For instance, we deduce whether two arguments are partial with the rule:

\noindent
\begin{lstlisting}[basicstyle=\ttfamily, xleftmargin=2em,  mathescape=true]
partial2(P,I,J):- I<J, type(P,I,T$_I$), type(P,J,T$_J$), domain(T$_I$,C$_I$), domain(T$_J$,C$_J$), 
                    not holds2(P,I,J,C$_I$,C$_J$).
\end{lstlisting}

\noindent
Finally, we deduce whether arguments are total by whether they are not partial.
We use an ASP solver to find a model of $\mathcal{P}$ and thus find total literals.
Algorithm \ref{alg:shrinker1} then returns a set of atoms denoting the total arguments of background literals.
For instance, for the relation \emph{length/2}, where the first argument is total, \name{} returns the atom \emph{total\_length\_b}.
For the relation \emph{add/3}, where any pair of variables is total, \name{} returns three atoms \emph{total\_add\_ab}, \emph{total\_add\_ac}, \emph{total\_add\_bc}.

\subsection{\name{} Correctness and Complexity}
We show that \name{} is correct in that it solves the hypothesis space reduction problem (Definition \ref{def:hypredprob2}):

\begin{proposition}[\textbf{\name{} correctness}]
\label{prop:correctness}
Let $I=(E, B, \mathcal{H})$ be an ILP problem, 
$h \in \mathcal{H}$ be an optimal hypothesis for $I$,
and $\mathcal{H}'$ be the hypothesis space resulting from applying \name{} to $\mathcal{H}$ given $B$.
Then $h \in \mathcal{H'}$.
\end{proposition}
\begin{proof}
Propositions~\ref{prop:sound_unsat2},~\ref{prop:sound_sat3},~\ref{prop:sound_sat4},~\&~\ref{prop:sound_sat5} imply that the pruning implemented by \name{} does not prune optimal hypotheses. 
Thus, if $h \in \mathcal{H}$ then $h \in \mathcal{H'}$.
\end{proof}

\noindent
The time complexity
of \name{} (Algorithm \ref{alg:shrinker}) is dominated by the call to Algorithm \ref{alg:shrinker1} (finding unsatisfiable and implication reducible rules).
For Algorithm~\ref{alg:shrinker1}, let $p$ be number of distinct predicate symbols in $BK$, 
$a$ be the maximum arity of any predicate symbol in $BK$, 
$v = max\_vars$,
and $m = max\_size$.
Then, ignoring the time limit (\emph{timeout}), 
the worst-case complexity of Algorithm \ref{alg:shrinker1} is bounded by $O((pv^{a})^{m})$.
In other words, the complexity grows exponentially in both the maximum arity ($a$) and the maximum template size ($m$).


\subsection{\popper{}}
The output of \name{} (Algorithm \ref{alg:shrinker}) is a set of pointless rules, specifically a set of atoms representing pointless rules.
Any ILP system could use these rules to prune the hypothesis space, such as Aleph’s rule pruning mechanism \cite{aleph}.
We use these rules to shrink the hypothesis space of the constraint-based ILP system \popper{} \cite{popper}, which can learn optimal and recursive hypotheses.
There are many \popper{} variants, including ones that learn from noisy data \cite{maxsynth}, learn higher-order programs \cite{hopper}, and learn from probabilistic data \cite{propper}.
Although \name{} will shrink the hypothesis space of all the variants, we describe the simplest version because it is sufficient to explain how we combine it with \name{}.

\popper{} takes as input BK, training examples, and a maximum hypothesis size and learns hypotheses as definite programs.
\popper{} searches for an optimal hypothesis which entails all the positive examples, no negative ones, and has minimal size.
\popper{} uses a generate, test, combine, and constrain loop to find an optimal hypothesis.
\popper{} starts with an initial ASP encoding $\mathcal{P}$.
The encoding $\mathcal{P}$ can be viewed as a \emph{generator} because each model (answer set) of $\mathcal{P}$ represents a hypothesis.
The encoding $\mathcal{P}$ uses head (\emph{hlit}/3) and body (\emph{blit/3}) literals to represent hypotheses.
The first argument of each literal is the rule id, the second is the predicate symbol, and the third is the literal variables, where \emph{0} represents \emph{A}, \emph{1} represents \emph{B}, etc.
For instance, consider the rule:

\begin{center}
\begin{tabular}{l}
\emph{second(A,B) $\leftarrow$ tail(A,C), head(C,B)} 
\end{tabular}
\end{center} 

\noindent
\popper{} represents this rule as the set:
\begin{center}
\begin{tabular}{l}
\emph{\{hlit(0,second,(0,1)), blit(0,tail,(0,2)), blit(0,head,(2,1))\}}
\end{tabular}
\end{center} 

\noindent
The encoding $\mathcal{P}$ contains choice rules for \emph{head/3} and \emph{body/3} literals:
\begin{lstlisting}[basicstyle=\ttfamily, mathescape=true]
{hlit(Rule,Pred,Vars)}:- rule(Rule), vars(Vars,Arity), head_pred(Pred,Arity).
{blit(Rule,Pred,Vars)}:- rule(Rule), vars(Vars,Arity), body_pred(Pred,Arity).
\end{lstlisting}

\noindent
In these rules, 
\emph{rule(Rule)} denotes rule indices, 
\emph{vars(Vars,Arity)} denotes possible variable tuples, 
and the literals \emph{head\_pred(Pred,Arity)} and \emph{body\_pred(Pred,Arity)} denote predicate symbols and arities that may appear in the head or body of a rule respectively.

In the \emph{generate stage}, \popper{} uses an ASP solver to find a model of $\mathcal{P}$ for a fixed hypothesis size (enforced via a cardinality constraint on the number of head and body literals).
If no model is found, \popper{} increases the hypothesis size and loops again.
If a model exists, \popper{} converts it to a hypothesis $h$.

In the \emph{test stage}, \popper{} uses Prolog to test $h$ on the training examples and BK.
If $h$ entails at least one positive example and no negative examples then \popper{} saves $h$ as a \emph{promising hypothesis}.

In the \emph{combine stage}, \popper{} searches for a combination (a union) of promising hypotheses that entails all the positive examples and has minimal size.
\popper{} formulates the search as a combinatorial optimisation problem, specifically as an ASP optimisation problem \cite{combo}.
If a combination exists, \popper{} saves it as the best hypothesis so far and updates the maximum hypothesis size.

In the \emph{constrain stage}, \popper{} uses $h$ to build constraints.
\popper{} adds these constraints to $\mathcal{P}$ to prune models and thus prune the hypothesis space.
For instance, if $h$ does not entail any positive example, \popper{} adds a constraint to prune its specialisations as they are guaranteed not to entail any positive example.
For instance, the following constraint prunes all specialisations (supersets) of the rule \emph{second(A,B) $\leftarrow$ tail(A,C), head(C,B)}:

\begin{lstlisting}[basicstyle=\ttfamily, xleftmargin=2em,  mathescape=true]
:- hlit(R,second,(0,1)), blit(R,tail,(0,2)), blit(R,head,(2,1)).
\end{lstlisting}

\noindent
\popper{} repeats this loop until it exhausts the models of $\mathcal{P}$ or exceeds a user-defined timeout.
\popper{} then returns the best hypothesis found.

\subsection{\name{} + \popper{}}
We use \name{} to shrink the hypothesis space of \popper{} before \popper{} searches for a hypothesis, i.e. before \popper{} starts its loop.
To do so, we allow \popper{} to reason about the unsatisfiable, implication reducible, recall reducible, and singleton reducible rules found by \name{}.
We provide examples.

\subsubsection{Unsatisfiable Rules}

If \name{} finds an unsatisfiable rule of the form \emph{P(A,B), Q(B,A)} then we add all the \emph{unsat\_ab\_ba(P,Q)} atoms and the following constraint to \popper{}:

\begin{lstlisting}[basicstyle=\ttfamily, xleftmargin=2em,  mathescape=true]
:- unsat_ab_ba(P,Q), blit(Rule,P,(A,B)), blit(Rule,Q,(B,A)).
\end{lstlisting}

\noindent
For instance, if \texttt{unsat\_ab\_ba(succ,succ)} holds, this constraint prunes all models (and thus hypotheses) that contain the literals \emph{blit(Rule,succ,(A,B))} and \emph{blit(Rule,succ,(B,A))}, including all variable substitutions for \emph{A}, \emph{B}, and \emph{Rule}.
For instance, for a learning task with an arbitrary head literal $h$, the constraint prunes the rule:

\begin{center}
\begin{tabular}{l}
\emph{h $\leftarrow$ succ(A,B), succ(B,C), succ(C,B)}
\end{tabular}
\end{center}

\noindent
If \name{} finds an unsatisfiable rule of the form \emph{P(A,B), Q(B,C), R(A,C)} (antitransitivity)  then we add all the \emph{unsat\_ab\_bc\_ac(P,Q,R)} atoms found by \name{} and the following constraint to \popper{}:

\begin{lstlisting}[basicstyle=\ttfamily, xleftmargin=2em,  mathescape=true]
:- unsat_ab_bc_ac(P,Q,R), blit(Rule,P,(A,B)), blit(Rule,Q,(B,C)), blit(Rule,R,(A,C)).
\end{lstlisting}

\subsubsection{Implication Reducible Rules}
If \name{} finds an implication reducible rule of the form \emph{P(A), Q(A)} then we add all the \emph{implies\_a\_a(P,Q)} atoms and the following constraint to \popper{}:
\begin{lstlisting}[basicstyle=\ttfamily, xleftmargin=2em,  mathescape=true]
:- implies_a_a(P,Q), blit(Rule,P,(A,)), blit(Rule,Q,(A,)).
\end{lstlisting}

\noindent
If \name{} finds an implication reducible rule of the form \emph{P(A,B), Q(B,C), R(A,C)} then we add the following constraint and all \emph{implies\_ab\_bc\_ac(P,Q,R)} atoms found by \name{}.
\begin{lstlisting}[basicstyle=\ttfamily, xleftmargin=2em,  mathescape=true]
:- implies_ab_bc_ac(P,Q,R), blit(Rule,P,(A,B)), blit(Rule,Q,(B,C)), blit(Rule,R,(A,C)).
\end{lstlisting}

\subsubsection{Recall Reducible Rules}

We use ASP aggregate constraints to succinctly express recall constraints.
For instance, if \name{} discovers that \emph{succ/2} is functional (where for any ground first argument there is at most 1 ground second argument), it represents this information as the atom \emph{recall\_succ\_a\_b\_1}.
Given this atom, we add this constraint to \popper{}:

\begin{lstlisting}[basicstyle=\ttfamily, xleftmargin=2em,  mathescape=true]
:- blit(Rule,succ,(V0,_)), #count{V1:blit(Rule,succ,(V0,V1))} > 1.
\end{lstlisting}

\noindent
This constraint prunes all models (and thus hypotheses) containing the literals \emph{blit(Rule, succ, (A,B))} and \emph{blit(Rule, succ, (A,C))} where \emph{B!=C}, i.e. this constraint prunes all rules with the body literals \emph{succ(A,B)} and \emph{succ(A,C)}, where \emph{B!=C}.
This constraint applies to all variable substitutions for \emph{A} and \emph{B} and all rules \emph{Rule}.
For instance, for a learning task with an arbitrary head literal $h$, the constraint prunes the rule:

\begin{center}
\begin{tabular}{l}
\emph{h $\leftarrow$ succ(B,C), succ(B,D)}
\end{tabular}
\end{center}

\noindent
If \name{} discovers that \emph{succ/2} is injective (for any ground second argument there is exactly 1 ground first argument), it represents this information as the atom \emph{recall\_succ\_b\_a\_1}.
Given this atom, we add this constraint to \popper{}:

\begin{lstlisting}[basicstyle=\ttfamily, xleftmargin=2em,  mathescape=true]
:- blit(Rule,succ,(_,V1)), #count{V0: blit(Rule,succ,(V0,V1))} > 1.
\end{lstlisting}

\noindent
For instance, the constraint prunes the rule:

\begin{center}
\begin{tabular}{l}
\emph{h $\leftarrow$ succ(A,B), succ(C,B)}
\end{tabular}
\end{center}

\noindent
If \name{} returns the recall atom \emph{recall\_q\_bc\_a\_2},  we add this constraint to \popper{}:

\begin{lstlisting}[basicstyle=\ttfamily, xleftmargin=2em,  mathescape=true]
:- blit(Rule,q,(_,V1,V2)), #count{V0: blit(Rule,q,(V0,V1,V2))} > 2.
\end{lstlisting}

\subsubsection{Singleton Reducible Rules}

To prune singleton reducible rules, we add this ASP rule to the generate encoding of \popper{}:

\begin{lstlisting}[basicstyle=\ttfamily, xleftmargin=2em,  mathescape=true]
singleton(Rule,V):- rule(Rule), var(V),
    #count{P,Vars : blit(Rule,P,Vars), vars(Vars,_), member(V,Vars)} == 1.
\end{lstlisting}

\noindent
In this rule, \emph{var(V)} denotes a variable (\emph{V}) allowed in a rule and \emph{member(V,Vars)} is true when the variable \emph{V} is in the tuple of variables \emph{Vars}.
This ASP rule allows us to identify variables that only appear once in a rule.
We use this rule to build ASP constraints to prune singleton reducible rules.
For instance, for the relation \emph{length/2}, where the first argument is total, \name{} returns the atom \emph{total\_length\_b}.
We therefore add this constraint to \popper{}:
\begin{lstlisting}[basicstyle=\ttfamily, xleftmargin=2em,  mathescape=true]
:- blit(Rule,length,(V0,V1)), singleton(Rule,V1).
\end{lstlisting}

\noindent
For the relation \emph{add/3}, where any pair of variables is total, \name{} returns three atoms \emph{total\_add\_ab}, \emph{total\_add\_ac}, \emph{total\_add\_bc}.
We therefore add three constraints to \popper{}:
\begin{lstlisting}[basicstyle=\ttfamily, xleftmargin=2em,  mathescape=true]
:- blit(Rule,add,(V0,V1,V2)), singleton(Rule,V1), singleton(Rule,V2).
:- blit(Rule,add,(V0,V1,V2)), singleton(Rule,V0), singleton(Rule,V2).
:- blit(Rule,add,(V0,V1,V2)), singleton(Rule,V0), singleton(Rule,V1).
\end{lstlisting}

\section{Experiments}
\label{sec:exp}
The purpose of \name{} is to shrink a hypothesis space without pruning optimal hypotheses.
Proposition \ref{prop:correctness} shows that \name{} does not prune optimal hypotheses.
However, it is unclear in practice what impact \name{} has on learning performance.
Therefore, our experiments\footnote{All code and experimental data are archived at \url{https://doi.org/10.5281/zenodo.19068412}} aim to answer the question:

\begin{description}
\item[Q1] Can \name{} reduce learning times?
\end{description}

\noindent
To answer \textbf{Q1} we compare the learning times of \popper{} with and without \name{}.

\name{} removes unsatisfiable, implication reducible, recall reducible, and singleton reducible rules from the hypothesis space.
To understand the impact of each type of rule, our experiments aim to answer the question:
\begin{description}
\item[Q2] Can removing each type of pointless rule reduce learning times?
\end{description}
To answer \textbf{Q2} we compare the learning times of \popper{} when using \name{} but only removing a single type of pointless rule from the hypothesis space.

\name{} takes a timeout as a parameter. 
To understand the impact of this parameter, our experiments aim to answer the question:
\begin{description}
\item[Q3] Can more \name{} time reduce learning times?
\end{description}

\noindent
To answer \textbf{Q3}, we compare the impact of \name{} when using 10 and 100 second timeouts.

\subsection{Experimental Setup}

\subsubsection{Generalisation Error}
Each learning task contains training and testing examples.
We use the training examples to train the ILP system to learn a hypothesis.
We test a hypothesis on the testing examples to see how they generalise to unseen data.
Given a hypothesis $h$, background knowledge $B$, and a set of test examples $E$, we define the following.
A true positive is a positive example in $E$ entailed by $h \cup B$. 
A true negative is a negative example in $E$ not entailed by $h \cup B$. 
We denote the number of true positives and true negatives as $tp(h)$ and $tn(h)$ respectively. 
We measure generalisation error in terms of \emph{balanced predictive accuracy}.
This measure handles imbalanced data by evaluating the average performance across both positive and negative classes:
\begin{align*}
balancedaccuracy(h) = \frac{1}{2}\left( \frac{tp(h)}{tp(h)+fn(h)}+\frac{tn(h)}{tn(h)+fp(h)} \right)
\end{align*}
Balanced accuracy is equivalent to standard accuracy when the hypothesis performs equally well on both classes or when the data is balanced.

\subsubsection{Settings}
We use \name{} with the settings  \emph{max\_size=3}, \emph{max\_vars=6}, \emph{batch\_size=1000}, and \emph{timeout=10}, except when we use a timeout of 100 seconds to answer \textbf{Q3}.
We allow \popper{} to learn rules with at most 6 variables and at most 10 body literals.
We use an AWS m6a.16xlarge instance to run experiments where each learning task uses a single core.

\subsubsection{Method}
\label{sec:expmethod}
\popper{} is an anytime approach.
We use the best hypothesis found by \popper{} given a maximum learning time of 60 minutes.
Learning time is the time \popper{} needs to learn and prove optimality of a hypothesis and terminate.
We repeat all the experiments 10 times.
We measure mean balanced accuracy and mean learning time.
We round times over one second to the nearest second.
We plot and report 95\% confidence intervals.
We use a Wilcoxon Signed-Rank Test to determine the statistical significance of any differences in the results.
Any subsequent reference to a significance test refers to this test.

\subsubsection{Datasets}
We use several datasets:

\textbf{1D-ARC.} This dataset \cite{onedarc} contains visual reasoning tasks inspired by the abstract reasoning corpus \cite{arc}.

\textbf{Alzheimer.} These real-world tasks 
\cite{DBLP:journals/ngc/KingSS95} involve learning rules describing four properties desirable for drug design against Alzheimer’s disease.

\textbf{IGGP.} In inductive general game playing (IGGP) \cite{iggp}, the task is to induce rules from game traces from the general game-playing competition \cite{ggp}.

\textbf{IMDB.}
We use a real-world dataset which contains relations between movies, actors, and directors \cite{imdb}. 

\textbf{List functions.} The list functions dataset \cite{ruleefficient} is designed to evaluate human and machine concept learning ability. 
The goal of each task is to identify a function that maps input lists to output lists, where list elements are natural numbers. 
We use a relational encoding \cite{decomp}.

\textbf{Trains.}
The goal is to find a hypothesis that distinguishes east and west trains \cite{michalski:trains}.

\textbf{Zendo.}
Zendo is a multiplayer game where players must discover a secret rule by building structures \cite{discopopper}.

\subsection{Experimental Results}

\subsubsection{Q1. Can \name{} Reduce Learning Times?}

Figure \ref{fig:q1_popper} shows the improvements in learning times with \name{}.
A significance test confirms $(p < 0.05)$ that \name{}  decreases learning time on 178/419 (42.5\%) tasks and increases learning time on 18/419 (4.3\%) tasks. There is no significant difference on the other tasks.
The mean decrease in learning time is 20.1 $\pm$ 3.3 minutes, corresponding to a 81.1\% decrease and a 5.3x speedup.
The mean increase in learning time is 1.2 $\pm$ 0.4 seconds, corresponding to a 60.5\% increase and a 1.6x slowdown. 
These are minimum improvements because \popper{} without \name{} often times out after 60 minutes.
With a longer timeout, we would likely see greater improvements.
The full results showing the times for every learning task are in Appendix D.

\begin{figure}[h!]
\centering
\begin{minipage}{0.49\textwidth}
\includegraphics{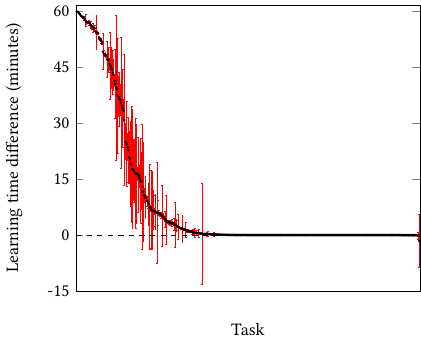}
\end{minipage}
\begin{minipage}{0.49\textwidth}
\includegraphics{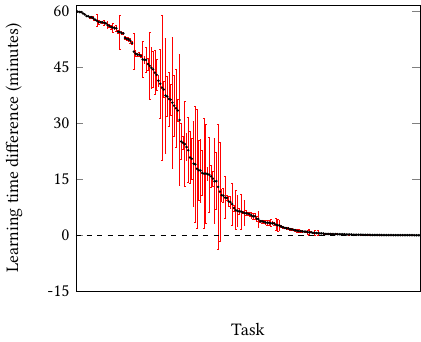}
  \end{minipage}
    \caption{
Learning time improvements when using \name{}.
The left figure shows all tasks and the right figure shows tasks where the two approaches significantly $(p < 0.05)$ differ.
The tasks are ordered by the improvement.
}
    \label{fig:q1_popper}
\end{figure}

\name{} can drastically reduce learning times.
For instance, for the \emph{list-function-043} task, where the task is to learn how to prepend a sequence of numbers to a list, \name{} reduces the learning time from $60\pm0$ minutes to $6 \pm 0$ minutes, a 90.2\% decrease and 10.2x speedup.
For the \emph{iggp-duikoshi-next\_control} task, \name{} reduces the learning time from $60\pm0$ minutes to $2\pm0$ seconds, a 99.9\% reduction and 1800x speedup.

To further explore the potential of \name{} to reduce learning times, we ran a separate experiment where we increased the maximum learning time to 24 hours for a subset of tasks.
With this greater timeout, for the \emph{iggp-duikoshi-next\_control} task, \name{} reduces the learning time from $6.5 \pm 1$ hours to $1 \pm 0$ seconds.
For the \emph{iggp-horseshoe-terminal} task, \name{} reduces the learning time from $10 \pm 0$ hours to $33 \pm 3$ seconds.

To illustrate why \name{} works, consider the \emph{iggp-scissors\_paper\_stone-next\_score} task.
The goal of this task is to learn the rules of the game \emph{rock, paper, scissors} from observations of gameplay. 
For this task, \name{} discovers that the background relation \emph{succ/2} is irreflexive, antitransitive, antitriangular, and asymmetric, which are expressed as unsatisfiable rules.
\name{} also discovers implication reducible rules, such as that the literals \emph{succ(A,B), int\_0(A), int\_1(B)} are implication reducible.
\name{} also discovers that the \emph{succ/2} relation is injective and functional, both expressed as recall reducible rules.
\name{} discovers these given only 10 seconds of preprocessing time, yet the resulting constraints reduce the learning time of \popper{} from $353\pm46$ seconds to $51\pm1$ seconds, an 85.5\% reduction and 6.9x speedup.

\name{} increases the learning time on 17/419 tasks.
The reason is the overhead of reasoning about the constraints discovered by \name{}, i.e. the overhead of Clingo processing the constraints.
For these tasks, the mean increase in learning time is only 1.2 $\pm$ 0.4 seconds.

A significance test confirms $(p < 0.05)$ that \name{} increases predictive accuracy on
15/419 (3.6\%) tasks and decreases accuracy on 24/419 (5.7\%) tasks.
There is no significant difference on the other tasks.
The main reason for any accuracy increase is that without \name{} \popper{} sometimes does not find a good hypothesis in the time limit. 
By contrast, as \name{} prunes the hypothesis space, there are fewer hypotheses for \popper{} to consider, so it sometimes finds a better hypothesis quicker.
Furthermore, as \name{} is optimally sound (Proposition \ref{prop:correctness}), it is guaranteed to lead to a hypothesis space that is a subset of the original one yet still contains an optimal hypothesis.

The main reason for the decrease in accuracy is a \popper{} implementation issue.
For the \emph{iggp-tiger\_vs\_dogs-goal} task, \popper{} without \name{} learns this rule:

\begin{lstlisting}[basicstyle=\ttfamily]   
goal(V0,V1,V2):- score_0(V2), agent_d(V1), pos_3(V5), mark_b(V3), pos_3(V4), 
                   true_cell(V0,V4,V5,V3).
\end{lstlisting}

\noindent
This rule has $58\% \pm 0$ accuracy.
\popper{} with \name{} cannot learn this rule.
The reason is that \name{} discovers that \emph{pos\_3(A)} has a maximum recall of 1, i.e. it is defined for a single value (3).
Therefore, \name{} says that this rule is recall reducible and prunes it from the hypothesis space.
A logically equivalent non-recall reducible rule is:
\begin{lstlisting}[basicstyle=\ttfamily]   
goal(V0,V1,V2):- score_0(V2), agent_d(V1), mark_b(V3), pos_3(V4), true_cell(V0,V4,V4,V3).
\end{lstlisting}
In this rule, the predicate symbol \emph{pos\_3} only appears once but the variable $V4$ appears twice in the literal \emph{true\_cell(V0,V4,V4,V3)}. 
\popper{} cannot learn this rule because it prohibits a variable from appearing twice in a literal (due to legacy reasons).
Therefore, \popper{} without \name{} uses multiple variables to learn a logically equivalent rule.

Overall, the results in this section show that the answer to \textbf{Q1} is yes, 
\name{} can drastically reduce learning times.
\subsubsection{Q2. Can Removing Each Type of Pointless Rule Reduce Learning Times?}

Figure \ref{fig:q2_unsat} shows the learning times when only removing unsatisfiable rules.
Doing so decreases learning time on 38/419 (9.1\%) tasks and increases learning time on 21/419 (5.0\%) tasks. There is no significant difference on the other tasks.
The mean decrease in learning time is 2.0 $\pm$ 1.0 minutes, corresponding to a 27.4\% decrease and a 1.4x speedup.
The mean increase in learning time is 2.8 $\pm$ 3.2 seconds, corresponding to a 46.0\% increase and a 1.5x slowdown. 

Figure \ref{fig:q2_impl} shows the learning times when only removing implication reducible rules.
Doing so decreases learning time on 154/419 (36.8\%) tasks and increases learning time on 10/419 (2.4\%) tasks. There is no significant difference on the other tasks.
The mean decrease in learning time is 19.2 $\pm$ 3.3 minutes, corresponding to a 73.1\% decrease and a 3.7x speedup.
The mean increase in learning time is 9.4 $\pm$ 18.5 seconds, corresponding to a 110.1\% increase and a 2.1x slowdown.

Figure \ref{fig:q2_recall} shows the learning times when only removing recall reducible rules.
Doing so decreases learning time on 70/419 (16.7\%) tasks and increases learning time on 9/419 (2.1\%) tasks. There is no significant difference on the other tasks.
The mean decrease in learning time is 3.5 $\pm$ 1.8 minutes, corresponding to a 38.8\% decrease and a 1.6x speedup.
The mean increase in learning time is 34.1 $\pm$ 68.7 seconds, corresponding to a 104.6\% increase and a 2.0x slowdown.
The reason for the increase in learning time is because, due to the aforementioned \emph{variable appearing twice in a literal} issue, \popper{} with \name{} sometimes needs to learn a longer rule, so it takes more time.

Figure \ref{fig:q2_singletons} shows the learning times when only removing singleton rules.
Doing so decreases learning time on 66/419 (15.8\%) tasks and increases learning time on 11/419 (2.6\%) tasks. There is no significant difference on the other tasks.
The mean decrease in learning time is 2.7 $\pm$ 1.1 minutes, corresponding to a 31.7\% decrease and a 1.5x speedup.
The mean increase in learning time is 1.3 $\pm$ 1.6 minutes, corresponding to a 25.9\% increase and a 1.3x slowdown.

Again, it is important to note these are minimum improvements, as \popper{} without \name{} often times out after 60 minutes.

Overall, the results in this section show that the answer to \textbf{Q2} is yes, removing each type of pointless rule can reduce learning times but removing implication reducible rules has the most impact.

\begin{figure}[h!]
\centering
\begin{minipage}{0.49\textwidth}
\includegraphics{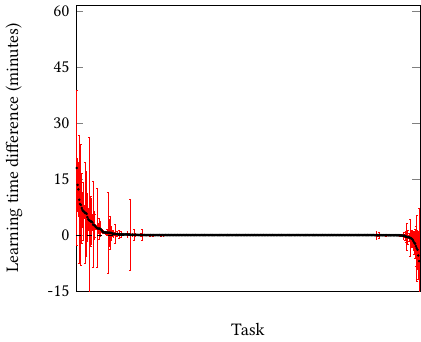}
\end{minipage}
\begin{minipage}{0.49\textwidth}
\includegraphics{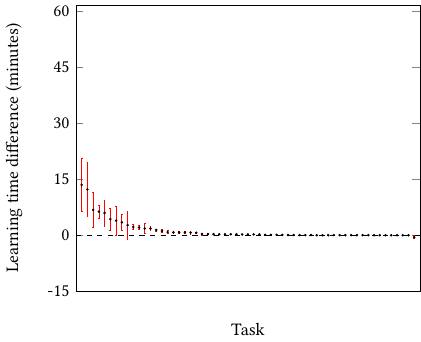}
  \end{minipage}
    \caption{
Learning time improvements when using \name{} with only unsatisfiable rules.
The left figure shows all tasks and the right figure shows tasks where the two approaches significantly $(p < 0.05)$ differ.
The tasks are ordered by the improvement.
}
    \label{fig:q2_unsat}
\end{figure}

\begin{figure}[h!]
\centering
\begin{minipage}{0.49\textwidth}
\includegraphics{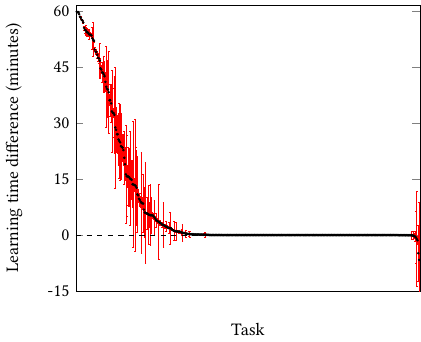}
\end{minipage}
\begin{minipage}{0.49\textwidth}
\includegraphics{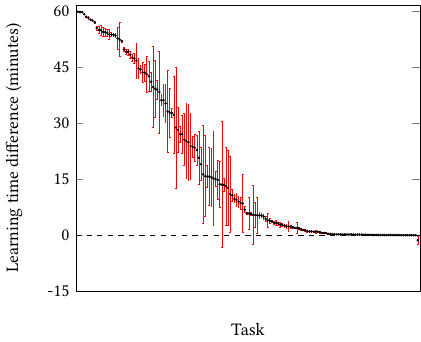}
  \end{minipage}
    \caption{
Learning time improvements when using \name{} with only implication reducible rules.
The left figure shows all tasks and the right figure shows tasks where the two approaches significantly $(p < 0.05)$ differ.
The tasks are ordered by the improvement.
}
    \label{fig:q2_impl}
\end{figure}

\begin{figure}[h!]
\centering
\begin{minipage}{0.49\textwidth}
\includegraphics{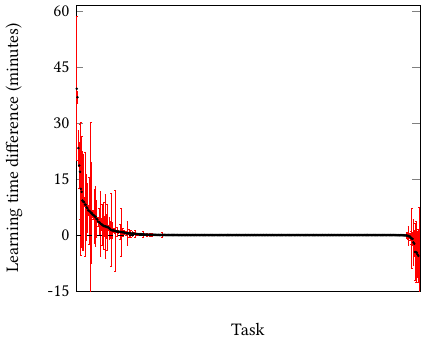}
\end{minipage}
\begin{minipage}{0.49\textwidth}
\includegraphics{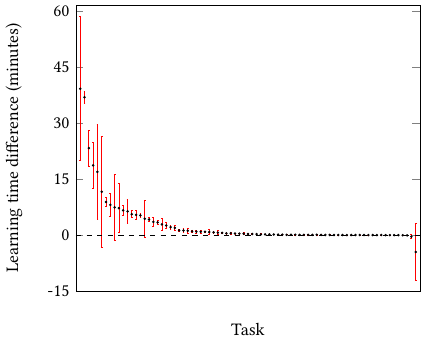}
  \end{minipage}
    \caption{
Learning time improvements when using \name{} with only recall reducible rules.
The left figure shows all tasks and the right figure shows tasks where the two approaches significantly $(p < 0.05)$ differ.
The tasks are ordered by the improvement.
}
    \label{fig:q2_recall}
\end{figure}

\begin{figure}[h!]
\centering
\begin{minipage}{0.49\textwidth}
\includegraphics{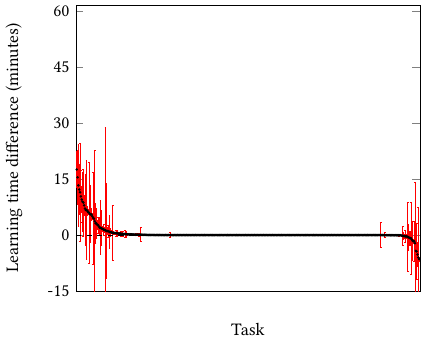}
\end{minipage}
\begin{minipage}{0.49\textwidth}
\includegraphics{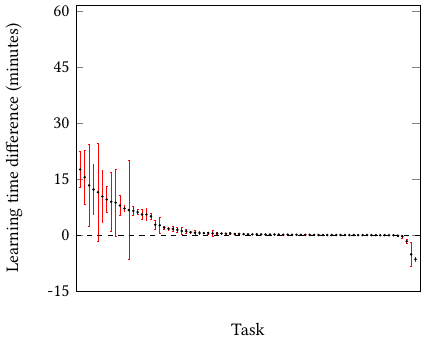}
  \end{minipage}
    \caption{
Learning time improvements when using \name{} with only singleton reducible rules.
The left figure shows all tasks and the right figure shows tasks where the two approaches significantly $(p < 0.05)$ differ.
The tasks are ordered by the improvement.
}
    \label{fig:q2_singletons}
\end{figure}

\subsubsection{Q3. Can More \name{} Time Reduce Learning Times?}

Figure \ref{fig:q3} shows the learning times when using \name{} with 10 and 100 second timeouts.
Significance tests confirm $(p < 0.05)$ that a 100 second timeout decreases learning time on 26/419 (6.2\%) tasks and increases learning time on 27/419 (6.4\%) tasks. There is no significant difference on the other tasks.
The mean decrease in learning time is 1.0 $\pm$ 0.9 minutes, corresponding to a 16.8\% decrease and a 1.2x speedup.
The mean increase in learning time is 4.0 $\pm$ 3.7 seconds, corresponding to a 36.9\% increase and a 1.4x slowdown. 
Overall, the results in this section show that the answer to \textbf{Q3} is yes, more \name{} time can improve performance but not commensurably so, and that 10 seconds is usually sufficient.

\begin{figure}[h!]
\centering
\begin{minipage}{0.49\textwidth}
\includegraphics{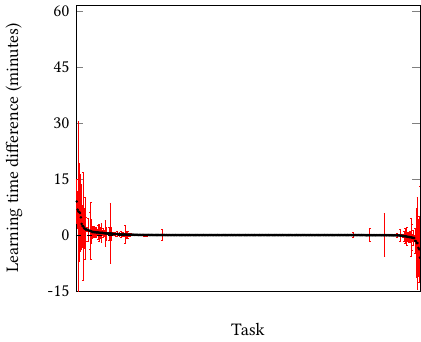}
\end{minipage}
\begin{minipage}{0.49\textwidth}
\includegraphics{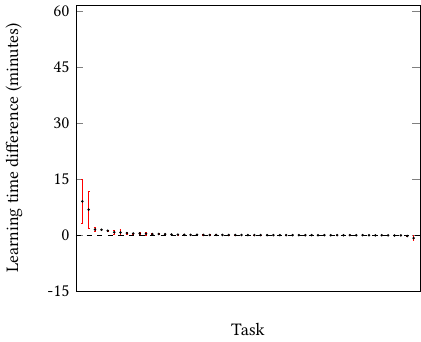}
  \end{minipage}
    \caption{
Learning time improvements when using \name{} with 100 vs 10 second timeout.
The left figure shows all tasks and the right figure shows tasks where the two approaches significantly $(p < 0.05)$ differ.
The tasks are ordered by the improvement.
}
    \label{fig:q3}
\end{figure}

\section{Conclusions and Limitations}

We introduced \name{}, an approach to automatically shrink the hypothesis space of an ILP system by preprocessing the background knowledge to find pointless rules.
\name{} finds four types of pointless rules (unsatisfiable, implication reducible, recall reducible, and singleton reducible).
We show that any hypothesis with a pointless rule is non-optimal so can be soundly removed from the hypothesis space (Propositions \ref{prop:sound_unsat2}, \ref{prop:sound_sat3}, \ref{prop:sound_sat4}, and \ref{prop:sound_sat5}).
Our experiments on multiple domains show that \name{} consistently reduces the learning times of an ILP system.
For instance, given only 10 seconds of preprocessing time, \name{} can reduce learning times from over 10 hours to 2 seconds.

\subsection*{Limitations}

\paragraph{Finite BK} 
Our shrinking idea is sufficiently general to handle definite programs as BK.
However, because our bottom-up implementation uses ASP, we require a finite grounding of the BK.
This restriction means that our implementation cannot handle BK with an infinite grounding, such as when reasoning about continuous values. 
Future work should address this limitation, such as by using top-down methods to discover pointless rules.

\paragraph{Closed-world assumption}
We adopt a closed-world assumption (CWA) to discover pointless rules from given BK.
For instance, we assume that \emph{odd(2)} does not hold if not given as BK.
As almost all ILP systems adopt a CWA, this limitation only applies if our approach is used with a system that does not make the CWA, such as recent rule mining algorithms \cite{rdfrules}.

\paragraph{Monotonic ILP}
Our theoretical results only hold for monotonic ILP, where a hypothesis is a definite program.
How to extend all of our results to non-monotonic ILP \cite{nopi}, such as learning programs with negation-as-failure \cite{naf}, is unclear.
The problem is that some pointless rules are no longer clearly pointless.
For instance, suppose a rule $r$ with the head literal \emph{aux} is unsatisfiable (Def \ref{def:unsatrule}).
Then an optimal hypothesis could still contain the body literal \emph{not aux}, where \emph{not} denotes negation-as-failure.
Extending the results to non-monotonic ILP is future work.


\paragraph{Noisy BK}
We assume that the BK is noiseless, i.e. if a fact is true in the BK then it is meant to be true.
Handling noisy BK is an open challenge \cite{ilpintro} that is beyond the scope of this paper.

\paragraph{Efficiency}
\name{} brute-force builds templates to find unsatisfiable and implication reducible rules.
However, we think that certain rules have a more significant impact on shrinking the hypothesis space.
Future work should, therefore, explore ways of ordering the templates to quickly find the most impactful rules.

\begin{acks}
Andrew Cropper and  Filipe Gouveia were supported by the EPSRC fellowship (\emph{EP/V040340/1}). 
David M. Cerna was supported by the Czech Science Foundation Grant 22-06414L and Cost Action CA20111 EuroProofNet.
\end{acks}


\printbibliography


\appendix
\section{Typing}
\label{sec:app_a}
We assume a simple typing mechanism with the properties  described below. 
\begin{definition}[\textbf{Types}]
Let $B$ be background knowledge. 
Then $types(B)$ is a set of unary predicate symbols defined in $B$ with the following properties: 
For all $ty_1,ty_2\in types(B)$ and constants $c$
\begin{itemize}
    \item  $c$ is of type $ty_1\in types(B)$ if and only if $B\models ty_1(c)$, and
    \item $ty_1(c)$ and $ty_2(c)$ if and only if $ty_1=ty_2$.
\end{itemize}
\end{definition}
\noindent For the rest of the paper we only consider background knowledge $B$ with non-empty $types(B)$.
\begin{definition}[\textbf{Position types}]
Let $B$ be background knowledge, 
$p$ be a predicate symbol with arity $a$, 
and $1\leq i\leq a$. Then $arg\_type_B(p,i)\in types(B)$ denotes the type of the $i^{th}$ argument of $p$ in $B$.
\end{definition}
\noindent For the rest of the paper we only consider predicate symbols whose arguments are typed with respect to the given background knowledge.
\begin{definition}[\textbf{Well-typed literal}]
Let $B$ be background knowledge and $l= p(t_1,\dots,t_a)$ a literal. Then $l$ is \emph{well-typed over $B$} if 
\begin{itemize}
    \item for all $1\leq i\leq a$ where $arg\_type(p,i)=ty$ and $t_i$ is a constant, $B\models ty(t_i)$ and
    \item for all $1\leq i,j\leq a$ where $t_i$ and $t_j$ are variables, $t_i=t_j$ iff  $arg\_type(p,i)=arg\_type(p,j)$.
\end{itemize}
\end{definition}
\noindent Using well-typed literals we can construct well typed-substitutions.
\begin{definition}[\textbf{Well-typed substitution}]
Let $B$ be background knowledge,  $l$ a literal, and $\sigma$ a substitution with domain $S\subseteq var(l)$. Then $\sigma$ is a \emph{well-typed substitution with respect to $l$} if $l\sigma$ is well-typed with respect to $B$. 
\end{definition}
\noindent In addition to assuming that all literals we consider are well-typed, we also assume the following consistency statement, i.e., non-well-typed literals do not entail from the background knowledge.
\begin{statement}[\textbf{Consistency}]
Let $B$ be background knowledge, $p$ a predicate symbol with arity $a$, $c_1,\dots,c_a$ constants, and for some $1\leq i\leq a$, $arg\_type(p,i)=ty$ and  $B\not\models ty(c_i)$. Then $B\not \models p(c_1,\dots ,c_a)$.
\end{statement}
\noindent Essentially, consistency states that what entails from the background knowledge is type consistent.
\begin{definition}[\textbf{Well-typed rule}]
Let $B$ be background knowledge and  $r$ a rule. Then $r$ is \emph{well-typed over $B$} if for all literals $p(t_1,\dots,t_a)$ and $q(s_1,\dots,s_b)$ in $r$, the following holds: for all $1\leq i\leq a$  and $1\leq j\leq b$ where $t_i$ and $s_j$ are variables, $t_i=s_j$ iff  $arg\_type(p,i)=arg\_type(q,j)$.
A hypothesis is well-typed over $B$ if all of its rules are.
\end{definition}
\noindent We assume all rules and hypotheses used in this paper are well-typed. Observe that the typing mechanisms outlined above can be enforced by adding the appropriate type literals to the body of a rule, thus no additional type checking infrastructure is needed. Note, we do not consider such type literals are part of the body and thus they do not count towards the size of a rule.

\section{Recall Reducible}
\label{app:recallsection}
Our proof of Proposition~\ref{prop:specRedundan} requires an alternative formulation of \textit{recall reducible} (Definition~\ref{def:Pigeonholedrule}) and then proving that the alternative formulation is equivalent to recall reducible ((Definition~\ref{def:recall_redundant}). This alternative formulation allows us to convert the substitution implicit in $\theta$-subsumption into equality literals. These equality literals are than used to test logical equivalence. The previously mentioned construction requires comparing argument tuples of a set of literals:

\begin{definition}[\textbf{Arguments}]
\label{def:args}
Let $l = p(X_1,\ldots,X_n) $ be a literal. 
Then $\emph{args}(l) = (X_1,\ldots,X_n)$, i.e. a tuple of variables containing the arguments of $l$ in the order they occur. 
\end{definition}

\noindent Our alternative definition of recall reducible requires identifying  mutually distinct argument tuples. 

\begin{definition}[\textbf{All different}]
Let $S$ be a set of tuples of variables such that for all $t_1,t_2\in S$, $|t_1|=|t_2|$. 
Then $\emph{alldiff}(S)= \{ neq(t_1,t_2)\mid (t_1,t_2\in S) \wedge (t_1\not = t_2)\}$ where for $t_1 = (X_1,\ldots,X_n)$ and  $t_2=(Y_1,\ldots,Y_n)$, we define the binary predicate $neq$ as follows: $neq(t_1,t_2) \equiv (X_1 \not = Y_1\vee \cdots \vee X_n \not = Y_n)$. 
\end{definition}

\noindent We can interpret this type of redundancy as adaptation of the \emph{pigeonhole principle} to rule learning.  The following definition captures this adaptation:

\begin{definition}[\textbf{Pigeonholed rule}]
\label{def:Pigeonholedrule}
Let 
$B$ be background knowledge,
$r$ be a rule,
$b \subseteq body(r)$ where every literal in $b$ has the same predicate symbol, 
$s= \{\emph{args}(l)\mid l\in b\}$,
and $B\not \models r\leftrightarrow (r\cup \{\mathit{alldiff}(s)\})$.
Then $r$ is pigeonholed.
\end{definition}

Essentially, Definition~\ref{def:Pigeonholedrule} states the following: if a rule $r$ has $n+1$ literals with the predicate symbol $p$, but can only instantiate the argument tuples of the literals in $k<n$ ways, then some of the literal occurrences are not necessary. The following examples illustrate this construction:

\begin{example}
Consider the following example and  Definition~\ref{def:Pigeonholedrule}


\begin{align*}
    B =& \{ edge(a,b), edge(b,c), edge(c,a)\}\\
    r =& h \leftarrow edge(X,Y), edge(Y,Z), edge(Z,W), edge(W,E) \\
    b =& \{edge(X,Y), edge(Y,Z), edge(Z,W), edge(W,E)\}\\
    s =& \{(X,Y),(Y,Z),(Z,W),(W,E)\}\\
    \mathit{alldiff}(s) =& \{ neq((X,Y),(Y,Z)), neq((X,Y),(Z,W)),\\ &\ neq((X,Y),(W,E)), neq((Y,Z),(Z,W)),\\ & \ neq((Y,Z),(W,E)), neq((Z,W),(W,E))\}
\end{align*}
Observe that $B\not \models r\leftrightarrow (r\cup \{\mathit{alldiff}(s)\})$ because $B$ defines a 3-cycle.
\end{example}

\begin{example}
Consider the following example and  Definition~\ref{def:Pigeonholedrule}  
\begin{align*}
    B =& \{ p(a,b,c), p(b,a,c), p(c,b,a)\}\\
    r =& h\leftarrow p(A,B,C),p(A,Y,Z).\\
    b =& \{p(A,B,C),p(A,Y,Z)\}\\
    s =& \{(A,B,C),(A,Y,Z)\}\\
    \mathit{alldiff}(s) =& \{ neq((A,B,C),(A,Y,Z))\}
\end{align*}
Observe that $B\not \models r\leftrightarrow (r\cup \{\mathit{alldiff}(s)\})$ because $B$ only contains one instance of $p/3$ per choice of first argument.
\end{example}

As we will show in Theorem~\ref{prop:PigeonholedRedundant}, \textit{recall reducible} (Definition~\ref{def:recall_redundant}) captures  our adaptation of the \emph{pigeonhole principle}.  Consider two rules $r_1$ and $r_2$. By Definition, $r_1$ is \textit{recall reducible} if $r_1\preceq_{\theta} r_2$, $|r_2|<|r_1|$ and the two rules are logically equivalent. Observe that this realises the inequalities introduced in Definition~\ref{def:Pigeonholedrule} through the implicit substitution of $\theta$-subsumption.

As mentioned above, recall reducible and Pigeonholed are equivalent concepts. While pigeonholed formalises our adaptation of the \emph{pigeonhole principle} through the introduction of inequalities,  recall reducible uses the implicit substitution of $\theta$-subsumption.  The proof of Theorem~\ref{prop:PigeonholedRedundant} provides the constructions allowing one to switch between the two concepts.
\begin{theorem}\label{prop:PigeonholedRedundant}
Let  $B$ be background knowledge  and  $r$ a rule. Then $r$ is recall reducible iff $r$ is pigeonholed.
\end{theorem}
\begin{proof}
\mbox{}\\
\noindent$\underline{\Longrightarrow:}$ By Definition~\ref{def:recall_redundant}, there exists $\theta$ such that $r\theta = r_2$. Furthermore, given that $|r|>|r_2|$, there must exists literals $l_1,l_2\in body(r_1)$ such that $l_1\not = l_2$ and $l_1\theta = l_2\theta$. Let $b$ be the subset of $body(r)$ containing all literals with the same predicate symbol as  $l_1$ and $l_2$. Observe that 
$$B\models (r\leftrightarrow r_2) \rightarrow \neg (r\leftrightarrow r\cup \mathit{alldiff}(\{\emph{args}(l)\mid l \in b\})),$$ holds because  the construction of $r_2$ unifies at least two literals in $b$. Thus, we derive that $B\models \neg (r\leftrightarrow r\cup \mathit{alldiff}(\{\emph{args}(l)\mid l \in b\})),$ i.e. $B\not \models r\leftrightarrow r\cup \mathit{alldiff}(\{\emph{args}(l)\mid l \in b\}).$ Thus, $r$ is pigeonholed.
\vspace{.5em}

\noindent$\underline{\Longleftarrow:}$ By Definition~\ref{def:Pigeonholedrule}, there exists $b\subseteq body(r)$ such that the literals in $b$ have the same predicate symbol and $B\not \models r\leftrightarrow (r\cup \{\mathit{alldiff}(s)\})$, where  $s= \{\emph{args}(l)\mid l\in b\}$. This implies that there exists $c\subset b$ and a mapping $f:c\rightarrow (c\setminus b)$  such that 
$$B \models r\leftrightarrow (r\cup \{\emph{args}(p)=\emph{args}(f(p))\mid p\in c\})$$
Now consider some $p\in c$. From $\emph{args}(p) =(X_1,\cdots, X_n)$ and $\emph{args}(f(p)) = (Y_1,\cdots, Y_n)$,  we construct a substitution $\theta = \{X_i\mapsto Y_i\mid 1\leq i\leq n\}$. Let $r_2 = r\theta$. Observe that $r\preceq_{\theta} r_2$, $|r|>|r_2|$, and $B\models r\leftrightarrow r_2$. Thus, r is recall reducible. 
\end{proof}
Given a recall reducible rule, specialisations of the rule are not necessarily recall reducible. However, rules containing the recall reducible rule as a subrule are recall reducible. This observation is not a simple corollary of the definition of recall reducible as proving it requires invoking Theorem~\ref{prop:PigeonholedRedundant} and performing most of the construction using pigeonholed rules. The full argument is presented below:

\renewcommand*{\theproposition}{5}
\begin{proposition}[\textbf{Recall specialisation}]
Let $B$ be BK, 
$r_1$ be \emph{recall reducible}, 
and $r_1\subseteq r_2$. 
Then $r_2$ is \emph{recall reducible}.
\end{proposition}
\begin{proof} 
By Theorem~\ref{prop:PigeonholedRedundant}, we know $r_1$ is also pigeonholed, thus 
there exists $b\subseteq body(r_1)$ such that all literals in $b$ have the same symbol and  $B\not \models r_1\leftrightarrow (r_1\cup \{\mathit{alldiff}(s)\})$, where $s= \{\emph{args}(l)\mid l\in b\}$. As in the proof of Theorem~\ref{prop:PigeonholedRedundant}, there exists $c\subset b$ and a mapping $f:c\rightarrow (c\setminus b)$  such that $B \models r_1\leftrightarrow (r_1\cup d_p),$
where
$$d_p=\{X_i=Y_i \mid  \emph{args}(p)= (X_1,\cdots, X_n)\ \wedge \ \emph{args}(f(p))= (Y_1,\cdots, Y_n)\ \wedge \ 1\leq i\leq n \}.$$ Furthermore, $B \models  (r_1\leftrightarrow (r_1\cup d_p))\rightarrow ((r_1\cup e)\leftrightarrow (r_1\cup d_p\cup e)),$
where $r_2=r_1\cup e$. Observe that $d_p$ can be replaced by a substitution $\theta =
 \{X \mapsto Y\mid X=Y\in d_p\}.$
Doing so results in the following: $B \models  (r_1\leftrightarrow r_1\theta)\rightarrow (r_2\leftrightarrow r_2\theta)$. Observe that $r_2\preceq r_2\theta$, $|r_2|> |r_2\theta|$, and given that $B \models  (r_1\leftrightarrow r_1\theta)$, we can derive that $B \models  (r_2\leftrightarrow r_2\theta)$. Thus, $r_2$ is recall reducible.
\end{proof}

\section{Implication vs Recall Reducible}
\label{sec:appendix_imp_recall_diff}

\noindent
The following example illustrates how \emph{implication} and \emph{recall} reducible rules differ.

\begin{example}
\label{ex:ImpRecExample}
We use the following BK with standard semantics:
\begin{center}
\begin{tabular}{l}
\emph{B = \{nat/1, odd/1, prime/1, succ/2, leq/2 \}}
\end{tabular}
\end{center}

\noindent
We consider these rules: 
\begin{center}
\begin{tabular}{l}
\emph{$r_1$ = h $\leftarrow$  leq(B,C), succ(A,B), succ(A,C)}\\
\emph{$r_2$ = h $\leftarrow$ nat(A), succ(A,B)}\\
\emph{$r_3$ = h $\leftarrow$ succ(A,B), succ(A,C), odd(B), prime(C)}\\
\end{tabular}
\end{center}

\noindent
We illustrate the relationships exemplified in the follow diagram

\begin{figure}[h!]
\centering
\includegraphics{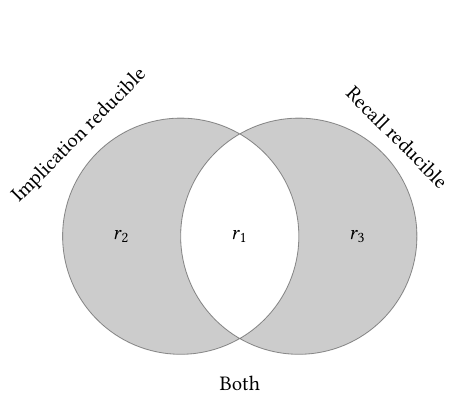}
\caption{The relationship between implication and recall reducible with respect to the rules introduced in Example~\ref{ex:ImpRecExample}}
\label{fig:vennImpRec}
 \end{figure}

\begin{itemize}
    \item[$\mathbf{r_1:}$]  Observe that for $r_4 = r_1\{C\mapsto B\}$  and $r_5= r_1\setminus \{leq(B,C)\}$
 \begin{align*}
    r_4 : \vspace{2em} h \mbox{ :- }& leq(B,B), succ(A,B).\\
    r_5 : \vspace{2em} h \mbox{ :- }& succ(A,B), succ(A,C).
    \end{align*}
 
  both $B\models r_1\leftrightarrow r_4$ and $B\models r_1\leftrightarrow r_5$ hold.  Furthermore $r_1 \preceq_{\theta} r_4$ and $|r_1|> |r_4|$ implying that $r_1$ is \emph{recall reducible} (Definition~\ref{def:recall_redundant}). For $r_5$, observe that $leq(B,C)$ is $r_1$-captured (Definition~\ref{def:capturedLiteral}), thus implying that $r$ is \emph{Implication reducible} (Definition~\ref{def:breducible}) as presented in Figure~\ref{fig:vennImpRec}. 

  \item [$\mathbf{r_2:}$] Observe that there does not exists a substitution $\theta$ such that $B\models r_2\leftrightarrow r_2\theta$ and $|r_2|> |r_2\theta|$. However, $r_6 = r_2\setminus \{nat(A)\}$ results in 
   \begin{align*}
    r_6 : \vspace{2em} h \mbox{ :- }&  succ(A,B).
   \end{align*}
    Observe that  $B\models r_2\leftrightarrow r_6$ as successor only applies to natural numbers. Thus we can deduce that $r_2$ is  \emph{Implication Reducible} (Definition~\ref{def:breducible}), but not  \emph{recall reducible} (Definition~\ref{def:recall_redundant}) as presented in Figure~\ref{fig:vennImpRec}.
\item [$\mathbf{r_3:}$] Observe that no matter which literal we remove from $r_3$ the meaning of the rule changes. For example, 
\begin{itemize}
    \item if we remove $succ(A,C)$, then $B$ need not be prime, 
    \item if we remove $succ(A,B)$, then $C$ need not be odd, i.e. 2 is prime,
    \item if we remove $odd(B)$, then $B$ and $C$ need not be odd, i.e. 2 is prime, and
    \item  if we remove $prime(C)$, then $B$ and $C$ may be divisible by 3.
\end{itemize}
Thus, $r_3$ cannot be \emph{Implication Reducible} (Definition~\ref{def:breducible}). However, 
 applying $\{B\mapsto C\}$ to $r_3$ results in the following rule:
    \begin{align*}
    r_7 : \vspace{2em} h \mbox{ :- }&  succ(A,C), odd(C), prime(C).
   \end{align*}
   Observe that  $|r_3|>|r_7|$ and $B\models r_3\leftrightarrow r_7$, thus implying that $r_3$ is \emph{recall reducible} (Definition~\ref{def:recall_redundant}) as presented in Figure~\ref{fig:vennImpRec}.
\end{itemize}

\end{example}

\section{Full results \popper{} vs \name{} results.}
\label{sec:app_full}

Table \ref{tab:full-results} shows the full experimental results using the method described in Section \ref{sec:expmethod}.
The \popper{} column denotes the learning times of \popper{} without \name{}.
The \name{} column denotes the learning times of \popper{} with \name{}.
The saving column shows the saving in learning time.

\begin{longtable}{llcccc}
    \caption{Full \popper{} vs \name{} learning times.}
    \label{tab:full-results}\\
\textbf{Domain} & \textbf{Task} & \textbf{\popper{}} & \textbf{\name{}} & \textbf{Saving} & \textbf{Speedup} \\
\hline
\endfirsthead

\textbf{Domain} & \textbf{Task} & \textbf{\popper{}} & \textbf{\name{}} & \textbf{Saving} & \textbf{Speedup} \\
\hline
\endhead
1d & 1d\_denoising\_1c & 1252 $\pm$ 332 & 253 $\pm$ 100 & 999 $\pm$ 366 & 4.95$\times$ \\
1d & 1d\_denoising\_mc & 1 $\pm$ 0 & 1 $\pm$ 0 & 0 $\pm$ 0 & 1.00$\times$ \\
1d & 1d\_fill & 18 $\pm$ 17 & 4 $\pm$ 3 & 14 $\pm$ 18 & 4.50$\times$ \\
1d & 1d\_flip & 3600 $\pm$ 0 & 2136 $\pm$ 681 & 1464 $\pm$ 681 & 1.69$\times$ \\
1d & 1d\_hollow & 3600 $\pm$ 0 & 1353 $\pm$ 926 & 2246 $\pm$ 926 & 2.66$\times$ \\
1d & 1d\_mirror & 928 $\pm$ 1000 & 153 $\pm$ 144 & 775 $\pm$ 1003 & 6.07$\times$ \\
1d & 1d\_move\_1p & 1 $\pm$ 0 & 1 $\pm$ 0 & 0 $\pm$ 0 & 1.00$\times$ \\
1d & 1d\_move\_2p & 1 $\pm$ 0 & 1 $\pm$ 0 & 0 $\pm$ 0 & 1.00$\times$ \\
1d & 1d\_move\_2p\_dp & 3350 $\pm$ 377 & 659 $\pm$ 303 & 2690 $\pm$ 334 & 5.08$\times$ \\
1d & 1d\_move\_3p & 1 $\pm$ 0 & 1 $\pm$ 0 & 0 $\pm$ 0 & 1.00$\times$ \\
1d & 1d\_move\_dp & 3399 $\pm$ 454 & 2998 $\pm$ 919 & 400 $\pm$ 631 & 1.13$\times$ \\
1d & 1d\_padded\_fill & 3600 $\pm$ 0 & 3600 $\pm$ 0 & 0 $\pm$ 0 & 1.00$\times$ \\
1d & 1d\_pcopy\_1c & 3600 $\pm$ 0 & 2732 $\pm$ 683 & 867 $\pm$ 683 & 1.32$\times$ \\
1d & 1d\_pcopy\_mc & 3549 $\pm$ 114 & 728 $\pm$ 189 & 2820 $\pm$ 196 & 4.88$\times$ \\
1d & 1d\_recolor\_cmp & 3600 $\pm$ 0 & 3600 $\pm$ 0 & 0 $\pm$ 0 & 1.00$\times$ \\
1d & 1d\_recolor\_cnt & 3600 $\pm$ 0 & 3600 $\pm$ 0 & 0 $\pm$ 0 & 1.00$\times$ \\
1d & 1d\_recolor\_oe & 3241 $\pm$ 812 & 3240 $\pm$ 812 & 0 $\pm$ 0 & 1.00$\times$ \\
1d & 1d\_scale\_dp & 212 $\pm$ 112 & 57 $\pm$ 28 & 155 $\pm$ 84 & 3.72$\times$ \\
alzheimer & acetyl & 3600 $\pm$ 0 & 1467 $\pm$ 1053 & 2132 $\pm$ 1053 & 2.45$\times$ \\
alzheimer & amine & 523 $\pm$ 353 & 618 $\pm$ 451 & -95 $\pm$ 430 & 0.85$\times$ \\
alzheimer & mem & 979 $\pm$ 313 & 375 $\pm$ 399 & 603 $\pm$ 340 & 2.61$\times$ \\
alzheimer & toxic & 1089 $\pm$ 762 & 1067 $\pm$ 784 & 22 $\pm$ 807 & 1.02$\times$ \\
iggp & alquerque-goal & 1 $\pm$ 0 & 1 $\pm$ 0 & 0 $\pm$ 0 & 1.00$\times$ \\
iggp & alquerque-legal\_move & 3600 $\pm$ 0 & 3600 $\pm$ 0 & 0 $\pm$ 0 & 1.00$\times$ \\
iggp & alquerque-next\_cell & 3600 $\pm$ 0 & 3600 $\pm$ 0 & 0 $\pm$ 0 & 1.00$\times$ \\
iggp & alquerque-next\_control & 2904 $\pm$ 287 & 76 $\pm$ 3 & 2827 $\pm$ 284 & 38.21$\times$ \\
iggp & alquerque-next\_score & 1088 $\pm$ 112 & 104 $\pm$ 4 & 984 $\pm$ 108 & 10.46$\times$ \\
iggp & alquerque-next\_step & 1 $\pm$ 0 & 2 $\pm$ 0 & 0 $\pm$ 0 & 0.50$\times$ \\
iggp & alquerque-terminal & 34 $\pm$ 2 & 6 $\pm$ 0 & 28 $\pm$ 2 & 5.67$\times$ \\
iggp & asylum-goal & 3 $\pm$ 0 & 4 $\pm$ 1 & -1 $\pm$ 1 & 0.75$\times$ \\
iggp & asylum-legal\_place & 3600 $\pm$ 0 & 2114 $\pm$ 105 & 1485 $\pm$ 105 & 1.70$\times$ \\
iggp & asylum-next\_color & 3600 $\pm$ 0 & 3600 $\pm$ 0 & 0 $\pm$ 0 & 1.00$\times$ \\
iggp & asylum-next\_control & 214 $\pm$ 25 & 26 $\pm$ 1 & 187 $\pm$ 25 & 8.23$\times$ \\
iggp & asylum-next\_location & 3600 $\pm$ 0 & 3600 $\pm$ 0 & 0 $\pm$ 0 & 1.00$\times$ \\
iggp & asylum-next\_step & 3 $\pm$ 0 & 3 $\pm$ 0 & 0 $\pm$ 0 & 1.00$\times$ \\
iggp & asylum-next\_strength & 3600 $\pm$ 0 & 1256 $\pm$ 135 & 2343 $\pm$ 135 & 2.87$\times$ \\
iggp & asylum-terminal & 1 $\pm$ 0 & 2 $\pm$ 0 & 0 $\pm$ 0 & 0.50$\times$ \\
iggp & battle\_of\_numbers-goal & 3600 $\pm$ 0 & 272 $\pm$ 43 & 3327 $\pm$ 43 & 13.24$\times$ \\
iggp & battle\_of\_numbers-legal\_move & 3600 $\pm$ 0 & 3600 $\pm$ 0 & 0 $\pm$ 0 & 1.00$\times$ \\
iggp & battle\_of\_numbers-next\_capture & 3600 $\pm$ 0 & 3600 $\pm$ 0 & 0 $\pm$ 0 & 1.00$\times$ \\
iggp & battle\_of\_numbers-next\_cell & 3600 $\pm$ 0 & 3600 $\pm$ 0 & 0 $\pm$ 0 & 1.00$\times$ \\
iggp & battle\_of\_numbers-next\_control & 3600 $\pm$ 0 & 3600 $\pm$ 0 & 0 $\pm$ 0 & 1.00$\times$ \\
iggp & battle\_of\_numbers-next\_step & 1 $\pm$ 0 & 2 $\pm$ 0 & 0 $\pm$ 0 & 0.50$\times$ \\
iggp & battle\_of\_numbers-terminal & 3600 $\pm$ 0 & 3600 $\pm$ 0 & 0 $\pm$ 0 & 1.00$\times$ \\
iggp & breakthrough-goal & 3600 $\pm$ 0 & 3600 $\pm$ 0 & 0 $\pm$ 0 & 1.00$\times$ \\
iggp & breakthrough-next\_cell & 3600 $\pm$ 0 & 3600 $\pm$ 0 & 0 $\pm$ 0 & 1.00$\times$ \\
iggp & breakthrough-next\_control & 3600 $\pm$ 0 & 986 $\pm$ 352 & 2613 $\pm$ 352 & 3.65$\times$ \\
iggp & breakthrough-terminal & 3600 $\pm$ 0 & 3600 $\pm$ 0 & 0 $\pm$ 0 & 1.00$\times$ \\
iggp & buttons\_and\_lights-goal & 1 $\pm$ 0 & 1 $\pm$ 0 & 0 $\pm$ 0 & 1.00$\times$ \\
iggp & buttons\_and\_lights-legal & 1 $\pm$ 0 & 1 $\pm$ 0 & 0 $\pm$ 0 & 1.00$\times$ \\
iggp & buttons\_and\_lights-terminal & 1 $\pm$ 0 & 1 $\pm$ 0 & 0 $\pm$ 0 & 1.00$\times$ \\
iggp & centipede-goal & 16 $\pm$ 1 & 1 $\pm$ 0 & 15 $\pm$ 1 & 16.00$\times$ \\
iggp & centipede-legal & 129 $\pm$ 16 & 3 $\pm$ 0 & 125 $\pm$ 16 & 43.00$\times$ \\
iggp & centipede-next\_at & 3600 $\pm$ 0 & 450 $\pm$ 26 & 3149 $\pm$ 26 & 8.00$\times$ \\
iggp & centipede-next\_dir & 557 $\pm$ 86 & 12 $\pm$ 1 & 544 $\pm$ 85 & 46.42$\times$ \\
iggp & centipede-terminal & 3600 $\pm$ 0 & 257 $\pm$ 72 & 3342 $\pm$ 72 & 14.01$\times$ \\
iggp & checkers-goal & 3 $\pm$ 0 & 5 $\pm$ 0 & -1 $\pm$ 0 & 0.60$\times$ \\
iggp & checkers-next\_capturecount & 3600 $\pm$ 0 & 3356 $\pm$ 400 & 244 $\pm$ 400 & 1.07$\times$ \\
iggp & checkers-next\_control & 3600 $\pm$ 0 & 845 $\pm$ 97 & 2754 $\pm$ 97 & 4.26$\times$ \\
iggp & checkers-next\_location & 3600 $\pm$ 0 & 3600 $\pm$ 0 & 0 $\pm$ 0 & 1.00$\times$ \\
iggp & checkers-next\_step & 3 $\pm$ 0 & 3 $\pm$ 0 & 0 $\pm$ 0 & 1.00$\times$ \\
iggp & checkers-terminal & 1749 $\pm$ 521 & 383 $\pm$ 57 & 1365 $\pm$ 522 & 4.57$\times$ \\
iggp & coins-goal & 1 $\pm$ 0 & 1 $\pm$ 0 & 0 $\pm$ 0 & 1.00$\times$ \\
iggp & coins-legal\_jump & 6 $\pm$ 0 & 3 $\pm$ 0 & 2 $\pm$ 0 & 2.00$\times$ \\
iggp & coins-next\_cell & 57 $\pm$ 4 & 24 $\pm$ 0 & 32 $\pm$ 3 & 2.38$\times$ \\
iggp & coins-next\_step & 5 $\pm$ 0 & 1 $\pm$ 0 & 4 $\pm$ 0 & 5.00$\times$ \\
iggp & coins-terminal & 20 $\pm$ 3 & 2 $\pm$ 0 & 17 $\pm$ 3 & 10.00$\times$ \\
iggp & connect4team-goal & 3600 $\pm$ 0 & 133 $\pm$ 6 & 3466 $\pm$ 6 & 27.07$\times$ \\
iggp & connect4team-legal\_drop & 3600 $\pm$ 0 & 3600 $\pm$ 0 & 0 $\pm$ 0 & 1.00$\times$ \\
iggp & connect4team-next\_cell & 3600 $\pm$ 0 & 3600 $\pm$ 0 & 0 $\pm$ 0 & 1.00$\times$ \\
iggp & connect4team-next\_control & 1435 $\pm$ 220 & 19 $\pm$ 1 & 1416 $\pm$ 221 & 75.53$\times$ \\
iggp & connect4team-terminal & 3600 $\pm$ 0 & 3600 $\pm$ 0 & 0 $\pm$ 0 & 1.00$\times$ \\
iggp & dont\_touch-goal & 273 $\pm$ 37 & 9 $\pm$ 0 & 264 $\pm$ 36 & 30.33$\times$ \\
iggp & dont\_touch-legal\_mark & 15 $\pm$ 0 & 8 $\pm$ 0 & 7 $\pm$ 1 & 1.88$\times$ \\
iggp & dont\_touch-next\_cell & 3600 $\pm$ 0 & 3600 $\pm$ 0 & 0 $\pm$ 0 & 1.00$\times$ \\
iggp & dont\_touch-next\_control & 3600 $\pm$ 0 & 11 $\pm$ 1 & 3588 $\pm$ 1 & 327.27$\times$ \\
iggp & dont\_touch-terminal & 3600 $\pm$ 0 & 3600 $\pm$ 0 & 0 $\pm$ 0 & 1.00$\times$ \\
iggp & duikoshi-goal & 88 $\pm$ 6 & 4 $\pm$ 0 & 84 $\pm$ 6 & 22.00$\times$ \\
iggp & duikoshi-legal\_mark & 3600 $\pm$ 0 & 95 $\pm$ 9 & 3504 $\pm$ 9 & 37.89$\times$ \\
iggp & duikoshi-next\_cell & 3600 $\pm$ 0 & 715 $\pm$ 16 & 2884 $\pm$ 16 & 5.03$\times$ \\
iggp & duikoshi-next\_control & 3600 $\pm$ 0 & 2 $\pm$ 0 & 3597 $\pm$ 0 & 1800.00$\times$ \\
iggp & duikoshi-terminal & 3600 $\pm$ 0 & 168 $\pm$ 32 & 3431 $\pm$ 32 & 21.43$\times$ \\
iggp & eight\_puzzle-goal & 66 $\pm$ 3 & 9 $\pm$ 0 & 57 $\pm$ 3 & 7.33$\times$ \\
iggp & eight\_puzzle-legal\_move & 349 $\pm$ 27 & 3 $\pm$ 0 & 345 $\pm$ 27 & 116.33$\times$ \\
iggp & eight\_puzzle-next\_cell & 3600 $\pm$ 0 & 695 $\pm$ 20 & 2904 $\pm$ 20 & 5.18$\times$ \\
iggp & eight\_puzzle-next\_step & 1 $\pm$ 0 & 1 $\pm$ 0 & 0 $\pm$ 0 & 1.00$\times$ \\
iggp & eight\_puzzle-terminal & 1 $\pm$ 0 & 1 $\pm$ 0 & 0 $\pm$ 0 & 1.00$\times$ \\
iggp & farming-goal & 3 $\pm$ 0 & 4 $\pm$ 1 & -1 $\pm$ 1 & 0.75$\times$ \\
iggp & farming-legal\_arson\_col & 615 $\pm$ 92 & 83 $\pm$ 8 & 531 $\pm$ 94 & 7.41$\times$ \\
iggp & farming-legal\_arson\_row & 677 $\pm$ 47 & 78 $\pm$ 11 & 599 $\pm$ 44 & 8.68$\times$ \\
iggp & farming-legal\_harvest\_col & 31 $\pm$ 2 & 14 $\pm$ 2 & 16 $\pm$ 3 & 2.21$\times$ \\
iggp & farming-legal\_harvest\_row & 28 $\pm$ 2 & 15 $\pm$ 2 & 12 $\pm$ 3 & 1.87$\times$ \\
iggp & farming-legal\_plow\_col & 28 $\pm$ 3 & 12 $\pm$ 3 & 16 $\pm$ 5 & 2.33$\times$ \\
iggp & farming-legal\_plow\_row & 28 $\pm$ 2 & 13 $\pm$ 4 & 14 $\pm$ 5 & 2.15$\times$ \\
iggp & farming-legal\_sow\_col & 9 $\pm$ 2 & 8 $\pm$ 0 & 1 $\pm$ 2 & 1.12$\times$ \\
iggp & farming-legal\_sow\_row & 9 $\pm$ 2 & 8 $\pm$ 1 & 0 $\pm$ 3 & 1.12$\times$ \\
iggp & farming-legal\_water\_col & 28 $\pm$ 3 & 15 $\pm$ 2 & 13 $\pm$ 3 & 1.87$\times$ \\
iggp & farming-legal\_water\_row & 29 $\pm$ 2 & 12 $\pm$ 2 & 16 $\pm$ 3 & 2.42$\times$ \\
iggp & farming-next\_control & 3600 $\pm$ 0 & 3600 $\pm$ 0 & 0 $\pm$ 0 & 1.00$\times$ \\
iggp & farming-next\_has\_arson & 446 $\pm$ 46 & 11 $\pm$ 0 & 435 $\pm$ 46 & 40.55$\times$ \\
iggp & farming-next\_plowed & 3600 $\pm$ 0 & 3600 $\pm$ 0 & 0 $\pm$ 0 & 1.00$\times$ \\
iggp & farming-next\_ripe & 137 $\pm$ 6 & 21 $\pm$ 1 & 116 $\pm$ 8 & 6.52$\times$ \\
iggp & farming-next\_score & 3435 $\pm$ 290 & 166 $\pm$ 19 & 3268 $\pm$ 273 & 20.69$\times$ \\
iggp & farming-next\_season & 3600 $\pm$ 0 & 3467 $\pm$ 200 & 132 $\pm$ 200 & 1.04$\times$ \\
iggp & farming-next\_sown & 3600 $\pm$ 0 & 3600 $\pm$ 0 & 0 $\pm$ 0 & 1.00$\times$ \\
iggp & farming-next\_step & 3 $\pm$ 0 & 4 $\pm$ 0 & 0 $\pm$ 0 & 0.75$\times$ \\
iggp & farming-next\_year\_first\_player & 3600 $\pm$ 0 & 3600 $\pm$ 0 & 0 $\pm$ 0 & 1.00$\times$ \\
iggp & farming-next\_year\_second\_player & 3600 $\pm$ 0 & 3600 $\pm$ 0 & 0 $\pm$ 0 & 1.00$\times$ \\
iggp & farming-terminal & 1 $\pm$ 0 & 2 $\pm$ 0 & 0 $\pm$ 0 & 0.50$\times$ \\
iggp & firesheep-goal & 1 $\pm$ 0 & 1 $\pm$ 0 & 0 $\pm$ 0 & 1.00$\times$ \\
iggp & firesheep-legal\_burn & 3600 $\pm$ 0 & 3600 $\pm$ 0 & 0 $\pm$ 0 & 1.00$\times$ \\
iggp & firesheep-legal\_force\_noop & 3600 $\pm$ 0 & 9 $\pm$ 0 & 3590 $\pm$ 0 & 400.00$\times$ \\
iggp & firesheep-legal\_freeze & 3600 $\pm$ 0 & 3600 $\pm$ 0 & 0 $\pm$ 0 & 1.00$\times$ \\
iggp & firesheep-legal\_kill & 3600 $\pm$ 0 & 3600 $\pm$ 0 & 0 $\pm$ 0 & 1.00$\times$ \\
iggp & firesheep-next\_at & 3600 $\pm$ 0 & 3600 $\pm$ 0 & 0 $\pm$ 0 & 1.00$\times$ \\
iggp & firesheep-next\_burning & 3600 $\pm$ 0 & 3600 $\pm$ 0 & 0 $\pm$ 0 & 1.00$\times$ \\
iggp & firesheep-next\_forced & 1 $\pm$ 0 & 1 $\pm$ 0 & 0 $\pm$ 0 & 1.00$\times$ \\
iggp & firesheep-next\_frozen & 3600 $\pm$ 0 & 3600 $\pm$ 0 & 0 $\pm$ 0 & 1.00$\times$ \\
iggp & firesheep-next\_grass & 3600 $\pm$ 0 & 3600 $\pm$ 0 & 0 $\pm$ 0 & 1.00$\times$ \\
iggp & firesheep-next\_grass\_last\_turn & 3600 $\pm$ 0 & 3600 $\pm$ 0 & 0 $\pm$ 0 & 1.00$\times$ \\
iggp & firesheep-next\_has\_force\_noop & 3600 $\pm$ 0 & 3600 $\pm$ 0 & 0 $\pm$ 0 & 1.00$\times$ \\
iggp & firesheep-next\_has\_kill & 3600 $\pm$ 0 & 268 $\pm$ 23 & 3331 $\pm$ 23 & 13.43$\times$ \\
iggp & firesheep-next\_score & 3600 $\pm$ 0 & 3600 $\pm$ 0 & 0 $\pm$ 0 & 1.00$\times$ \\
iggp & firesheep-next\_sheep & 3600 $\pm$ 0 & 1595 $\pm$ 172 & 2004 $\pm$ 172 & 2.26$\times$ \\
iggp & firesheep-terminal & 3600 $\pm$ 0 & 3600 $\pm$ 0 & 0 $\pm$ 0 & 1.00$\times$ \\
iggp & fizzbuzz-goal & 235 $\pm$ 25 & 14 $\pm$ 1 & 221 $\pm$ 24 & 16.79$\times$ \\
iggp & fizzbuzz-legal\_say & 3600 $\pm$ 0 & 934 $\pm$ 292 & 2665 $\pm$ 292 & 3.85$\times$ \\
iggp & fizzbuzz-next\_count & 1 $\pm$ 0 & 2 $\pm$ 0 & -1 $\pm$ 0 & 0.50$\times$ \\
iggp & fizzbuzz-next\_success & 3600 $\pm$ 0 & 3600 $\pm$ 0 & 0 $\pm$ 0 & 1.00$\times$ \\
iggp & fizzbuzz-terminal & 1 $\pm$ 0 & 1 $\pm$ 0 & 0 $\pm$ 0 & 1.00$\times$ \\
iggp & forager2-goal & 1 $\pm$ 0 & 1 $\pm$ 0 & 0 $\pm$ 0 & 1.00$\times$ \\
iggp & forager2-legal & 1 $\pm$ 0 & 1 $\pm$ 0 & 0 $\pm$ 0 & 1.00$\times$ \\
iggp & forager2-next\_at & 3600 $\pm$ 0 & 2450 $\pm$ 684 & 1149 $\pm$ 684 & 1.47$\times$ \\
iggp & forager2-next\_score & 3600 $\pm$ 0 & 3240 $\pm$ 814 & 359 $\pm$ 814 & 1.11$\times$ \\
iggp & forager2-next\_time & 1 $\pm$ 0 & 1 $\pm$ 0 & 0 $\pm$ 0 & 1.00$\times$ \\
iggp & forager2-terminal & 3240 $\pm$ 814 & 2171 $\pm$ 733 & 1069 $\pm$ 958 & 1.49$\times$ \\
iggp & freeforall-goal & 2 $\pm$ 0 & 1 $\pm$ 0 & 0 $\pm$ 0 & 2.00$\times$ \\
iggp & freeforall-next\_capture & 3600 $\pm$ 0 & 346 $\pm$ 25 & 3253 $\pm$ 25 & 10.40$\times$ \\
iggp & freeforall-next\_cell & 3600 $\pm$ 0 & 3600 $\pm$ 0 & 0 $\pm$ 0 & 1.00$\times$ \\
iggp & freeforall-next\_control & 1 $\pm$ 0 & 3 $\pm$ 0 & -2 $\pm$ 0 & 0.33$\times$ \\
iggp & freeforall-next\_step & 1 $\pm$ 0 & 1 $\pm$ 0 & 0 $\pm$ 0 & 1.00$\times$ \\
iggp & freeforall-terminal & 1 $\pm$ 0 & 2 $\pm$ 0 & -1 $\pm$ 0 & 0.50$\times$ \\
iggp & frogs\_and\_toads-legal\_jump & 159 $\pm$ 4 & 161 $\pm$ 8 & -2 $\pm$ 6 & 0.99$\times$ \\
iggp & frogs\_and\_toads-legal\_move & 155 $\pm$ 5 & 159 $\pm$ 5 & -3 $\pm$ 7 & 0.97$\times$ \\
iggp & frogs\_and\_toads-next\_cell & 3600 $\pm$ 0 & 3600 $\pm$ 0 & 0 $\pm$ 0 & 1.00$\times$ \\
iggp & frogs\_and\_toads-next\_correctFrogs & 285 $\pm$ 57 & 92 $\pm$ 7 & 192 $\pm$ 57 & 3.10$\times$ \\
iggp & frogs\_and\_toads-next\_correctToads & 527 $\pm$ 48 & 145 $\pm$ 13 & 381 $\pm$ 54 & 3.63$\times$ \\
iggp & frogs\_and\_toads-next\_step & 2 $\pm$ 0 & 2 $\pm$ 0 & 0 $\pm$ 0 & 1.00$\times$ \\
iggp & gt\_attrition-goal & 1 $\pm$ 0 & 1 $\pm$ 0 & 0 $\pm$ 0 & 1.00$\times$ \\
iggp & gt\_attrition-legal & 2 $\pm$ 0 & 1 $\pm$ 0 & 1 $\pm$ 0 & 2.00$\times$ \\
iggp & gt\_attrition-next\_claim\_made\_by & 1 $\pm$ 0 & 1 $\pm$ 0 & 0 $\pm$ 0 & 1.00$\times$ \\
iggp & gt\_attrition-next\_control & 1 $\pm$ 0 & 1 $\pm$ 0 & 0 $\pm$ 0 & 1.00$\times$ \\
iggp & gt\_attrition-next\_score & 2278 $\pm$ 317 & 47 $\pm$ 1 & 2230 $\pm$ 318 & 48.47$\times$ \\
iggp & gt\_attrition-terminal & 1 $\pm$ 0 & 1 $\pm$ 0 & 0 $\pm$ 0 & 1.00$\times$ \\
iggp & gt\_centipede-goal & 13 $\pm$ 1 & 1 $\pm$ 0 & 11 $\pm$ 1 & 13.00$\times$ \\
iggp & gt\_centipede-legal & 1 $\pm$ 0 & 1 $\pm$ 0 & 0 $\pm$ 0 & 1.00$\times$ \\
iggp & gt\_centipede-next\_blackPayoff & 7 $\pm$ 1 & 2 $\pm$ 0 & 4 $\pm$ 1 & 3.50$\times$ \\
iggp & gt\_centipede-next\_control & 1 $\pm$ 0 & 1 $\pm$ 0 & 0 $\pm$ 0 & 1.00$\times$ \\
iggp & gt\_centipede-next\_whitePayoff & 2 $\pm$ 0 & 1 $\pm$ 0 & 1 $\pm$ 0 & 2.00$\times$ \\
iggp & gt\_centipede-terminal & 1 $\pm$ 0 & 1 $\pm$ 0 & 0 $\pm$ 0 & 1.00$\times$ \\
iggp & gt\_chicken-goal & 5 $\pm$ 0 & 5 $\pm$ 0 & 0 $\pm$ 0 & 1.00$\times$ \\
iggp & gt\_chicken-legal & 125 $\pm$ 12 & 34 $\pm$ 3 & 91 $\pm$ 10 & 3.68$\times$ \\
iggp & gt\_chicken-next\_blackScore & 3600 $\pm$ 0 & 3600 $\pm$ 0 & 0 $\pm$ 0 & 1.00$\times$ \\
iggp & gt\_chicken-next\_rounds & 1 $\pm$ 0 & 1 $\pm$ 0 & 0 $\pm$ 0 & 1.00$\times$ \\
iggp & gt\_chicken-next\_whiteScore & 3600 $\pm$ 0 & 3600 $\pm$ 0 & 0 $\pm$ 0 & 1.00$\times$ \\
iggp & gt\_chicken-terminal & 1 $\pm$ 0 & 1 $\pm$ 0 & 0 $\pm$ 0 & 1.00$\times$ \\
iggp & gt\_prisoner-goal & 7 $\pm$ 0 & 5 $\pm$ 0 & 1 $\pm$ 0 & 1.40$\times$ \\
iggp & gt\_prisoner-legal & 145 $\pm$ 22 & 36 $\pm$ 5 & 108 $\pm$ 18 & 4.03$\times$ \\
iggp & gt\_prisoner-next\_blackScore & 3600 $\pm$ 0 & 3600 $\pm$ 0 & 0 $\pm$ 0 & 1.00$\times$ \\
iggp & gt\_prisoner-next\_maxRounds & 1 $\pm$ 0 & 1 $\pm$ 0 & 0 $\pm$ 0 & 1.00$\times$ \\
iggp & gt\_prisoner-next\_rounds & 1 $\pm$ 0 & 1 $\pm$ 0 & 0 $\pm$ 0 & 1.00$\times$ \\
iggp & gt\_prisoner-next\_whiteScore & 3600 $\pm$ 0 & 3600 $\pm$ 0 & 0 $\pm$ 0 & 1.00$\times$ \\
iggp & gt\_prisoner-terminal & 1 $\pm$ 0 & 1 $\pm$ 0 & 0 $\pm$ 0 & 1.00$\times$ \\
iggp & gt\_ultimatum-goal & 18 $\pm$ 1 & 14 $\pm$ 1 & 3 $\pm$ 0 & 1.29$\times$ \\
iggp & gt\_ultimatum-legal\_offer & 1 $\pm$ 0 & 1 $\pm$ 0 & 0 $\pm$ 0 & 1.00$\times$ \\
iggp & gt\_ultimatum-next\_blackScore & 3600 $\pm$ 0 & 3600 $\pm$ 0 & 0 $\pm$ 0 & 1.00$\times$ \\
iggp & gt\_ultimatum-next\_control & 332 $\pm$ 33 & 129 $\pm$ 9 & 203 $\pm$ 26 & 2.57$\times$ \\
iggp & gt\_ultimatum-next\_offered & 3600 $\pm$ 0 & 3600 $\pm$ 0 & 0 $\pm$ 0 & 1.00$\times$ \\
iggp & gt\_ultimatum-next\_rounds & 3600 $\pm$ 0 & 3600 $\pm$ 0 & 0 $\pm$ 0 & 1.00$\times$ \\
iggp & gt\_ultimatum-next\_whiteScore & 3600 $\pm$ 0 & 3600 $\pm$ 0 & 0 $\pm$ 0 & 1.00$\times$ \\
iggp & gt\_ultimatum-terminal & 2 $\pm$ 1 & 3 $\pm$ 0 & -1 $\pm$ 1 & 0.67$\times$ \\
iggp & hexforthree-goal & 3600 $\pm$ 0 & 427 $\pm$ 20 & 3172 $\pm$ 20 & 8.43$\times$ \\
iggp & hexforthree-legal\_place & 3600 $\pm$ 0 & 3600 $\pm$ 0 & 0 $\pm$ 0 & 1.00$\times$ \\
iggp & hexforthree-next\_cell & 10 $\pm$ 0 & 10 $\pm$ 0 & 0 $\pm$ 0 & 1.00$\times$ \\
iggp & hexforthree-next\_connected & 3600 $\pm$ 0 & 3600 $\pm$ 0 & 0 $\pm$ 0 & 1.00$\times$ \\
iggp & hexforthree-next\_control & 3600 $\pm$ 0 & 112 $\pm$ 26 & 3487 $\pm$ 26 & 32.14$\times$ \\
iggp & hexforthree-next\_owner & 22 $\pm$ 0 & 16 $\pm$ 1 & 5 $\pm$ 1 & 1.38$\times$ \\
iggp & hexforthree-next\_step & 3 $\pm$ 0 & 3 $\pm$ 0 & 0 $\pm$ 0 & 1.00$\times$ \\
iggp & hexforthree-terminal & 3600 $\pm$ 0 & 3600 $\pm$ 0 & 0 $\pm$ 0 & 1.00$\times$ \\
iggp & horseshoe-goal & 304 $\pm$ 41 & 2 $\pm$ 0 & 301 $\pm$ 41 & 152.00$\times$ \\
iggp & horseshoe-legal\_move & 3600 $\pm$ 0 & 191 $\pm$ 4 & 3408 $\pm$ 4 & 18.85$\times$ \\
iggp & horseshoe-next\_cell & 3600 $\pm$ 0 & 76 $\pm$ 2 & 3523 $\pm$ 2 & 47.37$\times$ \\
iggp & horseshoe-next\_control & 3600 $\pm$ 0 & 30 $\pm$ 0 & 3569 $\pm$ 0 & 120.00$\times$ \\
iggp & horseshoe-next\_step & 1 $\pm$ 0 & 1 $\pm$ 0 & 0 $\pm$ 0 & 1.00$\times$ \\
iggp & horseshoe-terminal & 3600 $\pm$ 0 & 47 $\pm$ 6 & 3552 $\pm$ 6 & 76.60$\times$ \\
iggp & hunter-goal & 1 $\pm$ 0 & 1 $\pm$ 0 & 0 $\pm$ 0 & 1.00$\times$ \\
iggp & hunter-legal\_move & 3600 $\pm$ 0 & 2095 $\pm$ 292 & 1505 $\pm$ 292 & 1.72$\times$ \\
iggp & hunter-next\_captures & 3600 $\pm$ 0 & 3600 $\pm$ 0 & 0 $\pm$ 0 & 1.00$\times$ \\
iggp & hunter-next\_cell & 3600 $\pm$ 0 & 3600 $\pm$ 0 & 0 $\pm$ 0 & 1.00$\times$ \\
iggp & hunter-next\_step & 1 $\pm$ 0 & 1 $\pm$ 0 & 0 $\pm$ 0 & 1.00$\times$ \\
iggp & hunter-terminal & 1 $\pm$ 0 & 1 $\pm$ 0 & 0 $\pm$ 0 & 1.00$\times$ \\
iggp & knights\_tour-goal & 1 $\pm$ 0 & 1 $\pm$ 0 & 0 $\pm$ 0 & 1.00$\times$ \\
iggp & knights\_tour-legal\_move & 3600 $\pm$ 0 & 1115 $\pm$ 85 & 2485 $\pm$ 85 & 3.23$\times$ \\
iggp & knights\_tour-next\_cell & 3600 $\pm$ 0 & 680 $\pm$ 50 & 2919 $\pm$ 50 & 5.29$\times$ \\
iggp & knights\_tour-next\_moveCount & 2 $\pm$ 1 & 1 $\pm$ 0 & 0 $\pm$ 1 & 2.00$\times$ \\
iggp & knights\_tour-terminal & 84 $\pm$ 26 & 94 $\pm$ 47 & -9 $\pm$ 57 & 0.89$\times$ \\
iggp & kono-goal & 1 $\pm$ 0 & 1 $\pm$ 0 & 0 $\pm$ 0 & 1.00$\times$ \\
iggp & kono-next\_cell & 3600 $\pm$ 0 & 3600 $\pm$ 0 & 0 $\pm$ 0 & 1.00$\times$ \\
iggp & kono-next\_control & 2212 $\pm$ 255 & 32 $\pm$ 2 & 2180 $\pm$ 253 & 69.12$\times$ \\
iggp & kono-next\_score & 2452 $\pm$ 439 & 260 $\pm$ 12 & 2192 $\pm$ 433 & 9.43$\times$ \\
iggp & kono-next\_step & 1 $\pm$ 0 & 2 $\pm$ 0 & -1 $\pm$ 0 & 0.50$\times$ \\
iggp & kono-terminal & 1 $\pm$ 0 & 1 $\pm$ 0 & 0 $\pm$ 0 & 1.00$\times$ \\
iggp & leafy-goal & 3600 $\pm$ 0 & 509 $\pm$ 24 & 3090 $\pm$ 24 & 7.07$\times$ \\
iggp & leafy-legal\_move & 3600 $\pm$ 0 & 73 $\pm$ 6 & 3526 $\pm$ 6 & 49.32$\times$ \\
iggp & leafy-next\_isplayer & 3 $\pm$ 0 & 3 $\pm$ 0 & 0 $\pm$ 0 & 1.00$\times$ \\
iggp & leafy-next\_leaf & 3600 $\pm$ 0 & 3600 $\pm$ 0 & 0 $\pm$ 0 & 1.00$\times$ \\
iggp & leafy-terminal & 3600 $\pm$ 0 & 3600 $\pm$ 0 & 0 $\pm$ 0 & 1.00$\times$ \\
iggp & lightboard-goal & 111 $\pm$ 15 & 1 $\pm$ 0 & 110 $\pm$ 15 & 111.00$\times$ \\
iggp & lightboard-legal\_toggle & 1 $\pm$ 0 & 1 $\pm$ 0 & 0 $\pm$ 0 & 1.00$\times$ \\
iggp & lightboard-next\_on & 3600 $\pm$ 0 & 3600 $\pm$ 0 & 0 $\pm$ 0 & 1.00$\times$ \\
iggp & lightboard-next\_step & 1 $\pm$ 0 & 1 $\pm$ 0 & 0 $\pm$ 0 & 1.00$\times$ \\
iggp & lightboard-terminal & 1 $\pm$ 0 & 1 $\pm$ 0 & 0 $\pm$ 0 & 1.00$\times$ \\
iggp & minimal\_decay-legal & 1 $\pm$ 0 & 1 $\pm$ 0 & 0 $\pm$ 0 & 1.00$\times$ \\
iggp & minimal\_decay-next\_value & 1 $\pm$ 0 & 1 $\pm$ 0 & 0 $\pm$ 0 & 1.00$\times$ \\
iggp & minimal\_even-goal & 1 $\pm$ 0 & 1 $\pm$ 0 & 0 $\pm$ 0 & 1.00$\times$ \\
iggp & minimal\_even-legal\_choose & 1 $\pm$ 0 & 1 $\pm$ 0 & 0 $\pm$ 0 & 1.00$\times$ \\
iggp & minimal\_even-next\_chosen & 1 $\pm$ 0 & 1 $\pm$ 0 & 0 $\pm$ 0 & 1.00$\times$ \\
iggp & minimal\_even-terminal & 1 $\pm$ 0 & 1 $\pm$ 0 & 0 $\pm$ 0 & 1.00$\times$ \\
iggp & multiplebuttonsandlights-goal & 1 $\pm$ 0 & 1 $\pm$ 0 & 0 $\pm$ 0 & 1.00$\times$ \\
iggp & multiplebuttonsandlights-legal\_a & 1 $\pm$ 0 & 1 $\pm$ 0 & 0 $\pm$ 0 & 1.00$\times$ \\
iggp & multiplebuttonsandlights-legal\_b & 1 $\pm$ 0 & 1 $\pm$ 0 & 0 $\pm$ 0 & 1.00$\times$ \\
iggp & multiplebuttonsandlights-legal\_c & 1 $\pm$ 0 & 1 $\pm$ 0 & 0 $\pm$ 0 & 1.00$\times$ \\
iggp & multiplebuttonsandlights-next\_p & 3600 $\pm$ 0 & 236 $\pm$ 26 & 3363 $\pm$ 26 & 15.25$\times$ \\
iggp & multiplebuttonsandlights-next\_q & 1 $\pm$ 0 & 1 $\pm$ 0 & 0 $\pm$ 0 & 1.00$\times$ \\
iggp & multiplebuttonsandlights-next\_step & 1 $\pm$ 0 & 1 $\pm$ 0 & 0 $\pm$ 0 & 1.00$\times$ \\
iggp & multiplebuttonsandlights-terminal & 1 $\pm$ 0 & 1 $\pm$ 0 & 0 $\pm$ 0 & 1.00$\times$ \\
iggp & nineboardtictactoe-goal & 11 $\pm$ 0 & 3 $\pm$ 0 & 8 $\pm$ 0 & 3.67$\times$ \\
iggp & nineboardtictactoe-legal\_play & 421 $\pm$ 33 & 84 $\pm$ 2 & 337 $\pm$ 33 & 5.01$\times$ \\
iggp & nineboardtictactoe-next\_control & 3 $\pm$ 0 & 1 $\pm$ 0 & 1 $\pm$ 0 & 3.00$\times$ \\
iggp & nineboardtictactoe-next\_currentboard & 3600 $\pm$ 0 & 3600 $\pm$ 0 & 0 $\pm$ 0 & 1.00$\times$ \\
iggp & nineboardtictactoe-next\_mark & 3600 $\pm$ 0 & 3600 $\pm$ 0 & 0 $\pm$ 0 & 1.00$\times$ \\
iggp & nineboardtictactoe-terminal & 3600 $\pm$ 0 & 3600 $\pm$ 0 & 0 $\pm$ 0 & 1.00$\times$ \\
iggp & pentago-goal & 3600 $\pm$ 0 & 93 $\pm$ 10 & 3506 $\pm$ 10 & 38.71$\times$ \\
iggp & pentago-legal\_place & 3600 $\pm$ 0 & 3600 $\pm$ 0 & 0 $\pm$ 0 & 1.00$\times$ \\
iggp & pentago-legal\_rotate & 1 $\pm$ 0 & 1 $\pm$ 0 & 0 $\pm$ 0 & 1.00$\times$ \\
iggp & pentago-next\_cellholds & 3600 $\pm$ 0 & 3600 $\pm$ 0 & 0 $\pm$ 0 & 1.00$\times$ \\
iggp & pentago-next\_placecontrol & 3600 $\pm$ 0 & 219 $\pm$ 57 & 3380 $\pm$ 57 & 16.44$\times$ \\
iggp & pentago-next\_rotatecontrol & 1 $\pm$ 0 & 1 $\pm$ 0 & 0 $\pm$ 0 & 1.00$\times$ \\
iggp & pentago-terminal & 3600 $\pm$ 0 & 3600 $\pm$ 0 & 0 $\pm$ 0 & 1.00$\times$ \\
iggp & pilgrimage-goal & 2613 $\pm$ 245 & 34 $\pm$ 2 & 2578 $\pm$ 244 & 76.85$\times$ \\
iggp & pilgrimage-legal\_move & 3600 $\pm$ 0 & 3600 $\pm$ 0 & 0 $\pm$ 0 & 1.00$\times$ \\
iggp & pilgrimage-legal\_raise & 3600 $\pm$ 0 & 3600 $\pm$ 0 & 0 $\pm$ 0 & 1.00$\times$ \\
iggp & pilgrimage-next\_builder & 3600 $\pm$ 0 & 3600 $\pm$ 0 & 0 $\pm$ 0 & 1.00$\times$ \\
iggp & pilgrimage-next\_cell & 3600 $\pm$ 0 & 3600 $\pm$ 0 & 0 $\pm$ 0 & 1.00$\times$ \\
iggp & pilgrimage-next\_control & 3600 $\pm$ 0 & 175 $\pm$ 9 & 3424 $\pm$ 9 & 20.57$\times$ \\
iggp & pilgrimage-next\_moves & 386 $\pm$ 14 & 20 $\pm$ 1 & 366 $\pm$ 14 & 19.30$\times$ \\
iggp & pilgrimage-next\_phase & 3600 $\pm$ 0 & 471 $\pm$ 38 & 3128 $\pm$ 38 & 7.64$\times$ \\
iggp & pilgrimage-next\_pilgrim & 3600 $\pm$ 0 & 3600 $\pm$ 0 & 0 $\pm$ 0 & 1.00$\times$ \\
iggp & platformjumpers-goal & 3600 $\pm$ 0 & 3600 $\pm$ 0 & 0 $\pm$ 0 & 1.00$\times$ \\
iggp & platformjumpers-legal\_col & 3600 $\pm$ 0 & 3600 $\pm$ 0 & 0 $\pm$ 0 & 1.00$\times$ \\
iggp & platformjumpers-legal\_row & 3600 $\pm$ 0 & 3600 $\pm$ 0 & 0 $\pm$ 0 & 1.00$\times$ \\
iggp & platformjumpers-next\_cell & 3600 $\pm$ 0 & 3600 $\pm$ 0 & 0 $\pm$ 0 & 1.00$\times$ \\
iggp & platformjumpers-next\_coled & 39 $\pm$ 2 & 38 $\pm$ 1 & 1 $\pm$ 2 & 1.03$\times$ \\
iggp & platformjumpers-next\_control & 1327 $\pm$ 793 & 338 $\pm$ 39 & 989 $\pm$ 795 & 3.93$\times$ \\
iggp & platformjumpers-next\_jumper & 3600 $\pm$ 0 & 3600 $\pm$ 0 & 0 $\pm$ 0 & 1.00$\times$ \\
iggp & platformjumpers-next\_rowed & 39 $\pm$ 2 & 36 $\pm$ 1 & 2 $\pm$ 3 & 1.08$\times$ \\
iggp & platformjumpers-terminal & 5 $\pm$ 0 & 8 $\pm$ 0 & -3 $\pm$ 0 & 0.62$\times$ \\
iggp & rainbow-goal & 11 $\pm$ 0 & 1 $\pm$ 0 & 10 $\pm$ 0 & 11.00$\times$ \\
iggp & rainbow-legal\_mark & 321 $\pm$ 44 & 2 $\pm$ 0 & 318 $\pm$ 44 & 160.50$\times$ \\
iggp & rainbow-next\_color & 1 $\pm$ 0 & 1 $\pm$ 0 & 0 $\pm$ 0 & 1.00$\times$ \\
iggp & rainbow-terminal & 1252 $\pm$ 419 & 2 $\pm$ 0 & 1249 $\pm$ 419 & 626.00$\times$ \\
iggp & scissors\_paper\_stone-goal & 1 $\pm$ 0 & 1 $\pm$ 0 & 0 $\pm$ 0 & 1.00$\times$ \\
iggp & scissors\_paper\_stone-legal & 1 $\pm$ 0 & 1 $\pm$ 0 & 0 $\pm$ 0 & 1.00$\times$ \\
iggp & scissors\_paper\_stone-next\_score & 353 $\pm$ 46 & 51 $\pm$ 1 & 302 $\pm$ 45 & 6.92$\times$ \\
iggp & scissors\_paper\_stone-next\_step & 1 $\pm$ 0 & 1 $\pm$ 0 & 0 $\pm$ 0 & 1.00$\times$ \\
iggp & scissors\_paper\_stone-terminal & 1 $\pm$ 0 & 1 $\pm$ 0 & 0 $\pm$ 0 & 1.00$\times$ \\
iggp & sheep\_and\_wolf-goal & 3 $\pm$ 0 & 2 $\pm$ 0 & 0 $\pm$ 0 & 1.50$\times$ \\
iggp & sheep\_and\_wolf-next\_cell & 3600 $\pm$ 0 & 3600 $\pm$ 0 & 0 $\pm$ 0 & 1.00$\times$ \\
iggp & sheep\_and\_wolf-next\_control & 3600 $\pm$ 0 & 778 $\pm$ 58 & 2821 $\pm$ 58 & 4.63$\times$ \\
iggp & sheep\_and\_wolf-terminal & 3600 $\pm$ 0 & 3600 $\pm$ 0 & 0 $\pm$ 0 & 1.00$\times$ \\
iggp & sokoban-goal & 4 $\pm$ 0 & 1 $\pm$ 0 & 3 $\pm$ 0 & 4.00$\times$ \\
iggp & sokoban-legal & 235 $\pm$ 17 & 12 $\pm$ 0 & 223 $\pm$ 17 & 19.58$\times$ \\
iggp & sokoban-next\_at & 3600 $\pm$ 0 & 2372 $\pm$ 336 & 1227 $\pm$ 336 & 1.52$\times$ \\
iggp & sokoban-next\_target & 1 $\pm$ 0 & 1 $\pm$ 0 & 0 $\pm$ 0 & 1.00$\times$ \\
iggp & sokoban-terminal & 3600 $\pm$ 0 & 647 $\pm$ 288 & 2952 $\pm$ 288 & 5.56$\times$ \\
iggp & sudoku-goal & 7 $\pm$ 0 & 8 $\pm$ 0 & 0 $\pm$ 0 & 0.88$\times$ \\
iggp & sudoku-legal\_mark & 3600 $\pm$ 0 & 442 $\pm$ 21 & 3158 $\pm$ 21 & 8.14$\times$ \\
iggp & sudoku-next\_cell & 3600 $\pm$ 0 & 3600 $\pm$ 0 & 0 $\pm$ 0 & 1.00$\times$ \\
iggp & sudoku-terminal & 3600 $\pm$ 0 & 3600 $\pm$ 0 & 0 $\pm$ 0 & 1.00$\times$ \\
iggp & sukoshi-goal & 1 $\pm$ 0 & 1 $\pm$ 0 & 0 $\pm$ 0 & 1.00$\times$ \\
iggp & sukoshi-legal\_mark & 38 $\pm$ 2 & 1 $\pm$ 0 & 36 $\pm$ 2 & 38.00$\times$ \\
iggp & sukoshi-next\_cell & 472 $\pm$ 41 & 120 $\pm$ 3 & 351 $\pm$ 40 & 3.93$\times$ \\
iggp & sukoshi-terminal & 3058 $\pm$ 318 & 333 $\pm$ 419 & 2724 $\pm$ 538 & 9.18$\times$ \\
iggp & switches-legal & 96 $\pm$ 10 & 2 $\pm$ 0 & 94 $\pm$ 10 & 48.00$\times$ \\
iggp & switches-next\_at & 3600 $\pm$ 0 & 252 $\pm$ 29 & 3347 $\pm$ 29 & 14.29$\times$ \\
iggp & switches-next\_open & 8 $\pm$ 1 & 1 $\pm$ 0 & 6 $\pm$ 1 & 8.00$\times$ \\
iggp & switches-next\_switch & 1 $\pm$ 0 & 1 $\pm$ 0 & 0 $\pm$ 0 & 1.00$\times$ \\
iggp & switches-next\_target & 1 $\pm$ 0 & 1 $\pm$ 0 & 0 $\pm$ 0 & 1.00$\times$ \\
iggp & tictactoe-goal & 58 $\pm$ 6 & 1 $\pm$ 0 & 57 $\pm$ 6 & 58.00$\times$ \\
iggp & tictactoe-legal\_mark & 1 $\pm$ 0 & 1 $\pm$ 0 & 0 $\pm$ 0 & 1.00$\times$ \\
iggp & tictactoe-next\_cell & 3600 $\pm$ 0 & 325 $\pm$ 28 & 3274 $\pm$ 28 & 11.08$\times$ \\
iggp & tictactoe-next\_control & 946 $\pm$ 179 & 1 $\pm$ 0 & 945 $\pm$ 179 & 946.00$\times$ \\
iggp & tictactoe-terminal & 3600 $\pm$ 0 & 183 $\pm$ 34 & 3416 $\pm$ 34 & 19.67$\times$ \\
iggp & tiger\_vs\_dogs-goal & 19 $\pm$ 1 & 3 $\pm$ 0 & 15 $\pm$ 1 & 6.33$\times$ \\
iggp & tiger\_vs\_dogs-legal\_move & 3600 $\pm$ 0 & 3600 $\pm$ 0 & 0 $\pm$ 0 & 1.00$\times$ \\
iggp & tiger\_vs\_dogs-next\_cell & 3600 $\pm$ 0 & 3600 $\pm$ 0 & 0 $\pm$ 0 & 1.00$\times$ \\
iggp & tiger\_vs\_dogs-next\_control & 1 $\pm$ 0 & 1 $\pm$ 0 & 0 $\pm$ 0 & 1.00$\times$ \\
iggp & tron-goal & 147 $\pm$ 17 & 10 $\pm$ 0 & 137 $\pm$ 17 & 14.70$\times$ \\
iggp & tron-legal & 934 $\pm$ 125 & 24 $\pm$ 2 & 909 $\pm$ 123 & 38.92$\times$ \\
iggp & tron-next\_at & 3600 $\pm$ 0 & 697 $\pm$ 32 & 2902 $\pm$ 32 & 5.16$\times$ \\
iggp & tron-next\_marked & 1 $\pm$ 0 & 1 $\pm$ 0 & 0 $\pm$ 0 & 1.00$\times$ \\
iggp & tron-terminal & 1 $\pm$ 0 & 1 $\pm$ 0 & 0 $\pm$ 0 & 1.00$\times$ \\
iggp & ttcc4-goal & 3600 $\pm$ 0 & 175 $\pm$ 6 & 3424 $\pm$ 6 & 20.57$\times$ \\
iggp & ttcc4-legal\_checkermove & 3600 $\pm$ 0 & 3600 $\pm$ 0 & 0 $\pm$ 0 & 1.00$\times$ \\
iggp & ttcc4-legal\_drop & 3600 $\pm$ 0 & 3600 $\pm$ 0 & 0 $\pm$ 0 & 1.00$\times$ \\
iggp & ttcc4-legal\_knightmove & 3600 $\pm$ 0 & 3600 $\pm$ 0 & 0 $\pm$ 0 & 1.00$\times$ \\
iggp & ttcc4-legal\_pawnmove & 3600 $\pm$ 0 & 3600 $\pm$ 0 & 0 $\pm$ 0 & 1.00$\times$ \\
iggp & ttcc4-next\_cell & 3600 $\pm$ 0 & 3600 $\pm$ 0 & 0 $\pm$ 0 & 1.00$\times$ \\
iggp & ttcc4-next\_control & 3 $\pm$ 0 & 3 $\pm$ 0 & 0 $\pm$ 0 & 1.00$\times$ \\
iggp & ttcc4-next\_step & 2 $\pm$ 0 & 4 $\pm$ 0 & -2 $\pm$ 0 & 0.50$\times$ \\
iggp & ttcc4-terminal & 3600 $\pm$ 0 & 3600 $\pm$ 0 & 0 $\pm$ 0 & 1.00$\times$ \\
iggp & untwisty\_corridor-goal & 1 $\pm$ 0 & 1 $\pm$ 0 & 0 $\pm$ 0 & 1.00$\times$ \\
iggp & untwisty\_corridor-legal & 1 $\pm$ 0 & 1 $\pm$ 0 & 0 $\pm$ 0 & 1.00$\times$ \\
iggp & untwisty\_corridor-next\_step & 1 $\pm$ 0 & 1 $\pm$ 0 & 0 $\pm$ 0 & 1.00$\times$ \\
iggp & walkabout-goal & 81 $\pm$ 13 & 3 $\pm$ 0 & 78 $\pm$ 13 & 27.00$\times$ \\
iggp & walkabout-legal & 646 $\pm$ 99 & 3 $\pm$ 0 & 642 $\pm$ 99 & 215.33$\times$ \\
iggp & walkabout-next\_at & 3555 $\pm$ 99 & 106 $\pm$ 3 & 3449 $\pm$ 98 & 33.54$\times$ \\
iggp & walkabout-terminal & 7 $\pm$ 2 & 1 $\pm$ 0 & 6 $\pm$ 2 & 7.00$\times$ \\
imdb & imdb1 & 1 $\pm$ 0 & 1 $\pm$ 0 & 0 $\pm$ 0 & 1.00$\times$ \\
imdb & imdb2 & 1 $\pm$ 0 & 1 $\pm$ 0 & 0 $\pm$ 0 & 1.00$\times$ \\
imdb & imdb3 & 50 $\pm$ 31 & 34 $\pm$ 36 & 15 $\pm$ 36 & 1.47$\times$ \\
jr & 001 & 2 $\pm$ 1 & 2 $\pm$ 1 & 0 $\pm$ 1 & 1.00$\times$ \\
jr & 002 & 2 $\pm$ 0 & 1 $\pm$ 0 & 0 $\pm$ 0 & 2.00$\times$ \\
jr & 003 & 3 $\pm$ 1 & 2 $\pm$ 0 & 1 $\pm$ 1 & 1.50$\times$ \\
jr & 004 & 3 $\pm$ 1 & 2 $\pm$ 0 & 0 $\pm$ 1 & 1.50$\times$ \\
jr & 005 & 2 $\pm$ 1 & 1 $\pm$ 0 & 0 $\pm$ 1 & 2.00$\times$ \\
jr & 006 & 1 $\pm$ 0 & 1 $\pm$ 0 & 0 $\pm$ 0 & 1.00$\times$ \\
jr & 007 & 1 $\pm$ 0 & 1 $\pm$ 0 & 0 $\pm$ 0 & 1.00$\times$ \\
jr & 008 & 3 $\pm$ 1 & 1 $\pm$ 0 & 1 $\pm$ 1 & 3.00$\times$ \\
jr & 009 & 3 $\pm$ 0 & 1 $\pm$ 0 & 1 $\pm$ 0 & 3.00$\times$ \\
jr & 010 & 3551 $\pm$ 109 & 1509 $\pm$ 617 & 2041 $\pm$ 628 & 2.35$\times$ \\
jr & 011 & 28 $\pm$ 10 & 10 $\pm$ 2 & 17 $\pm$ 8 & 2.80$\times$ \\
jr & 012 & 16 $\pm$ 4 & 7 $\pm$ 1 & 9 $\pm$ 4 & 2.29$\times$ \\
jr & 013 & 1057 $\pm$ 428 & 190 $\pm$ 121 & 867 $\pm$ 460 & 5.56$\times$ \\
jr & 014 & 532 $\pm$ 251 & 141 $\pm$ 45 & 391 $\pm$ 244 & 3.77$\times$ \\
jr & 015 & 3600 $\pm$ 0 & 3600 $\pm$ 0 & 0 $\pm$ 0 & 1.00$\times$ \\
jr & 016 & 218 $\pm$ 35 & 21 $\pm$ 2 & 196 $\pm$ 33 & 10.38$\times$ \\
jr & 017 & 1087 $\pm$ 513 & 46 $\pm$ 9 & 1040 $\pm$ 505 & 23.63$\times$ \\
jr & 018 & 242 $\pm$ 23 & 26 $\pm$ 2 & 215 $\pm$ 22 & 9.31$\times$ \\
jr & 019 & 680 $\pm$ 115 & 44 $\pm$ 2 & 635 $\pm$ 115 & 15.45$\times$ \\
jr & 020 & 318 $\pm$ 124 & 61 $\pm$ 25 & 257 $\pm$ 128 & 5.21$\times$ \\
jr & 021 & 74 $\pm$ 22 & 9 $\pm$ 1 & 64 $\pm$ 21 & 8.22$\times$ \\
jr & 022 & 68 $\pm$ 18 & 9 $\pm$ 2 & 59 $\pm$ 17 & 7.56$\times$ \\
jr & 023 & 3600 $\pm$ 0 & 3148 $\pm$ 682 & 451 $\pm$ 682 & 1.14$\times$ \\
jr & 024 & 3600 $\pm$ 0 & 3600 $\pm$ 0 & 0 $\pm$ 0 & 1.00$\times$ \\
jr & 025 & 40 $\pm$ 4 & 11 $\pm$ 1 & 28 $\pm$ 3 & 3.64$\times$ \\
jr & 026 & 485 $\pm$ 298 & 85 $\pm$ 59 & 400 $\pm$ 313 & 5.71$\times$ \\
jr & 027 & 3600 $\pm$ 0 & 3588 $\pm$ 26 & 11 $\pm$ 26 & 1.00$\times$ \\
jr & 028 & 3600 $\pm$ 0 & 3190 $\pm$ 625 & 409 $\pm$ 625 & 1.13$\times$ \\
jr & 029 & 1 $\pm$ 0 & 2 $\pm$ 1 & 0 $\pm$ 1 & 0.50$\times$ \\
jr & 030 & 97 $\pm$ 52 & 68 $\pm$ 26 & 29 $\pm$ 61 & 1.43$\times$ \\
jr & 031 & 3600 $\pm$ 0 & 3600 $\pm$ 0 & 0 $\pm$ 0 & 1.00$\times$ \\
jr & 032 & 3600 $\pm$ 0 & 3600 $\pm$ 0 & 0 $\pm$ 0 & 1.00$\times$ \\
jr & 033 & 1146 $\pm$ 604 & 174 $\pm$ 74 & 971 $\pm$ 606 & 6.59$\times$ \\
jr & 034 & 229 $\pm$ 26 & 40 $\pm$ 1 & 189 $\pm$ 24 & 5.72$\times$ \\
jr & 035 & 74 $\pm$ 17 & 11 $\pm$ 1 & 62 $\pm$ 16 & 6.73$\times$ \\
jr & 036 & 3600 $\pm$ 0 & 3600 $\pm$ 0 & 0 $\pm$ 0 & 1.00$\times$ \\
jr & 037 & 2 $\pm$ 0 & 1 $\pm$ 0 & 1 $\pm$ 0 & 2.00$\times$ \\
jr & 038 & 1 $\pm$ 0 & 1 $\pm$ 0 & 0 $\pm$ 0 & 1.00$\times$ \\
jr & 039 & 1977 $\pm$ 1067 & 123 $\pm$ 26 & 1854 $\pm$ 1057 & 16.07$\times$ \\
jr & 040 & 382 $\pm$ 202 & 27 $\pm$ 3 & 355 $\pm$ 200 & 14.15$\times$ \\
jr & 041 & 1 $\pm$ 0 & 1 $\pm$ 0 & 0 $\pm$ 0 & 1.00$\times$ \\
jr & 042 & 227 $\pm$ 44 & 29 $\pm$ 6 & 198 $\pm$ 39 & 7.83$\times$ \\
jr & 043 & 3600 $\pm$ 0 & 352 $\pm$ 15 & 3247 $\pm$ 15 & 10.23$\times$ \\
jr & 044 & 3600 $\pm$ 0 & 3600 $\pm$ 0 & 0 $\pm$ 0 & 1.00$\times$ \\
jr & 045 & 1 $\pm$ 0 & 1 $\pm$ 0 & 0 $\pm$ 0 & 1.00$\times$ \\
jr & 046 & 10 $\pm$ 1 & 1 $\pm$ 0 & 9 $\pm$ 1 & 10.00$\times$ \\
jr & 047 & 3600 $\pm$ 0 & 3600 $\pm$ 0 & 0 $\pm$ 0 & 1.00$\times$ \\
jr & 048 & 1 $\pm$ 0 & 1 $\pm$ 0 & 0 $\pm$ 0 & 1.00$\times$ \\
jr & 049 & 1 $\pm$ 0 & 1 $\pm$ 0 & 0 $\pm$ 0 & 1.00$\times$ \\
jr & 050 & 6 $\pm$ 0 & 1 $\pm$ 0 & 5 $\pm$ 0 & 6.00$\times$ \\
jr & 051 & 451 $\pm$ 296 & 72 $\pm$ 30 & 379 $\pm$ 281 & 6.26$\times$ \\
jr & 052 & 48 $\pm$ 33 & 13 $\pm$ 5 & 34 $\pm$ 34 & 3.69$\times$ \\
jr & 053 & 59 $\pm$ 8 & 15 $\pm$ 1 & 44 $\pm$ 8 & 3.93$\times$ \\
jr & 054 & 82 $\pm$ 22 & 12 $\pm$ 1 & 70 $\pm$ 20 & 6.83$\times$ \\
jr & 055 & 1548 $\pm$ 839 & 558 $\pm$ 304 & 990 $\pm$ 884 & 2.77$\times$ \\
jr & 056 & 801 $\pm$ 802 & 103 $\pm$ 41 & 698 $\pm$ 792 & 7.78$\times$ \\
jr & 057 & 271 $\pm$ 113 & 113 $\pm$ 66 & 158 $\pm$ 100 & 2.40$\times$ \\
jr & 058 & 3 $\pm$ 1 & 2 $\pm$ 1 & 0 $\pm$ 2 & 1.50$\times$ \\
jr & 059 & 3600 $\pm$ 0 & 3600 $\pm$ 0 & 0 $\pm$ 0 & 1.00$\times$ \\
jr & 060 & 1090 $\pm$ 491 & 46 $\pm$ 4 & 1044 $\pm$ 487 & 23.70$\times$ \\
jr & 061 & 69 $\pm$ 44 & 23 $\pm$ 6 & 45 $\pm$ 49 & 3.00$\times$ \\
jr & 062 & 3 $\pm$ 1 & 2 $\pm$ 0 & 1 $\pm$ 1 & 1.50$\times$ \\
jr & 063 & 3600 $\pm$ 0 & 3439 $\pm$ 362 & 160 $\pm$ 362 & 1.05$\times$ \\
jr & 064 & 106 $\pm$ 51 & 69 $\pm$ 32 & 37 $\pm$ 71 & 1.54$\times$ \\
jr & 065 & 34 $\pm$ 7 & 7 $\pm$ 1 & 26 $\pm$ 6 & 4.86$\times$ \\
jr & 066 & 624 $\pm$ 355 & 126 $\pm$ 74 & 498 $\pm$ 335 & 4.95$\times$ \\
jr & 067 & 2797 $\pm$ 577 & 362 $\pm$ 101 & 2435 $\pm$ 549 & 7.73$\times$ \\
jr & 068 & 3600 $\pm$ 0 & 3600 $\pm$ 0 & 0 $\pm$ 0 & 1.00$\times$ \\
jr & 069 & 3600 $\pm$ 0 & 3600 $\pm$ 0 & 0 $\pm$ 0 & 1.00$\times$ \\
jr & 070 & 6 $\pm$ 2 & 2 $\pm$ 0 & 3 $\pm$ 2 & 3.00$\times$ \\
jr & 071 & 9 $\pm$ 3 & 4 $\pm$ 1 & 4 $\pm$ 4 & 2.25$\times$ \\
jr & 072 & 161 $\pm$ 130 & 80 $\pm$ 34 & 81 $\pm$ 115 & 2.01$\times$ \\
jr & 073 & 2 $\pm$ 1 & 1 $\pm$ 0 & 1 $\pm$ 2 & 2.00$\times$ \\
jr & 074 & 3600 $\pm$ 0 & 3600 $\pm$ 0 & 0 $\pm$ 0 & 1.00$\times$ \\
jr & 075 & 81 $\pm$ 39 & 35 $\pm$ 18 & 45 $\pm$ 37 & 2.31$\times$ \\
jr & 076 & 3600 $\pm$ 0 & 3600 $\pm$ 0 & 0 $\pm$ 0 & 1.00$\times$ \\
jr & 077 & 1 $\pm$ 0 & 1 $\pm$ 0 & 0 $\pm$ 0 & 1.00$\times$ \\
jr & 078 & 3600 $\pm$ 0 & 2461 $\pm$ 937 & 1138 $\pm$ 937 & 1.46$\times$ \\
jr & 079 & 3600 $\pm$ 0 & 1515 $\pm$ 494 & 2084 $\pm$ 494 & 2.38$\times$ \\
jr & 080 & 514 $\pm$ 524 & 205 $\pm$ 145 & 308 $\pm$ 488 & 2.51$\times$ \\
trains & trains1 & 3 $\pm$ 2 & 2 $\pm$ 1 & 1 $\pm$ 1 & 1.50$\times$ \\
trains & trains2 & 17 $\pm$ 10 & 2 $\pm$ 0 & 15 $\pm$ 10 & 8.50$\times$ \\
trains & trains3 & 191 $\pm$ 100 & 25 $\pm$ 27 & 166 $\pm$ 95 & 7.64$\times$ \\
trains & trains4 & 2523 $\pm$ 1239 & 149 $\pm$ 90 & 2374 $\pm$ 1170 & 16.93$\times$ \\
zendo & zendo1 & 41 $\pm$ 43 & 10 $\pm$ 5 & 31 $\pm$ 41 & 4.10$\times$ \\
zendo & zendo2 & 3600 $\pm$ 0 & 3600 $\pm$ 0 & 0 $\pm$ 0 & 1.00$\times$ \\
zendo & zendo3 & 3600 $\pm$ 0 & 3583 $\pm$ 36 & 16 $\pm$ 36 & 1.00$\times$ \\
zendo & zendo4 & 3600 $\pm$ 0 & 3499 $\pm$ 220 & 100 $\pm$ 220 & 1.03$\times$ \\
\end{longtable}

\appendix

\end{document}